\documentclass[onecolumn]{article}

\usepackage[a4paper, total={6in, 8.5in}]{geometry}
\usepackage{amsmath, amsfonts, amssymb, amsthm}
\usepackage{graphicx}
\usepackage{authblk}
\usepackage{titlesec}
\titleformat*{\section}{\large\bfseries}
\titleformat*{\subsection}{\large\bfseries}

\usepackage{algorithm}
\usepackage{algorithmic}
\usepackage{amsthm}
\usepackage{subfigure}
\usepackage{epsfig}
\usepackage{graphicx}
\usepackage{url}
\usepackage{multirow}
\usepackage{amssymb}
\usepackage{amsthm}
\usepackage{mathrsfs}
\usepackage{epstopdf}
\usepackage{mathtools}
\usepackage{multicol}

\newtheorem{theorem}{\textbf{Theorem}}
\newtheorem{assumption}{\textbf{Assumption}}
\newtheorem{lemma}{\textbf{Lemma}}
\newtheorem{corollary}{\textbf{Corollary}}
\newtheorem{remark}{\textbf{Remark}}
\newtheorem{definition}{\textbf{Definition}}

\newcommand\blfootnote[1]{%
	\begingroup
	\renewcommand\thefootnote{}\footnote{#1}%
	\addtocounter{footnote}{-1}%
	\endgroup
}

\title{Periodic Stochastic Gradient Descent with Momentum for Decentralized Training}
\author[1]{Hongchang Gao}
\author[2]{Heng Huang}

\affil[1]{Temple University}
\affil[2]{University of Pittsburgh}

\date{}

\begin{document}

\maketitle

\begin{abstract}
	Decentralized training has been  actively studied in recent years. Although  a wide variety of methods have been proposed, yet the decentralized  momentum SGD  method is still underexplored. In this paper, we propose a novel periodic decentralized momentum SGD method, which employs the momentum schema  and periodic communication for decentralized training. With these two strategies, as well as the topology of the decentralized training system, the theoretical convergence analysis of our proposed  method is difficult. We address this challenging  problem and  provide the condition  under which our proposed method can achieve the linear speedup regarding the number of  workers. Furthermore,  we also  introduce a communication-efficient variant to  reduce  the communication cost in each communication round. The condition for achieving the linear  speedup  is also provided for this  variant.  To the  best of our  knowledge, these two  methods are all the first ones  achieving these theoretical results in  their corresponding domain. We conduct extensive experiments to  verify the performance  of our proposed  two methods, and both of  them  have shown  superior performance over existing methods. 
\end{abstract}

\blfootnote{hongchanggao@gmail.com, heng.huang@pitt.edu}

\section{Introduction}
Decentralized training of large-scale non-convex machine learning models has attracted increasing attention in recent years with the emergence of the distributed data. Formally, the decentralized training is to optimize the following machine learning model:
\begin{equation} \label{lossfunction}
\begin{aligned}
&	\min_{\mathbf{x}} f(\mathbf{x}) = \frac{1}{K}\sum_{k=1}^K f^{(k)}(\mathbf{x}) \ , \\
\end{aligned}
\end{equation}
where $\mathbf{x}\in \Re^d$ stands for the model parameter,  $K$ represents the number of workers in a decentralized training system, $f^{(k)}(\mathbf{x})\triangleq\mathbb{E}_{\xi\sim\mathcal{D}^{(k)}} F^{(k)}(\mathbf{x}; \xi)$  denotes the local non-convex loss function on the $k$-th worker and $\mathcal{D}^{(k)}$  is the data distribution of the $k$-th worker. In this decentralized training system, each worker optimizes the model based on its own data and communicates the learned model parameter with its neighboring workers.

In recent years, a wide variety of decentralized training methods \cite{koloskova2019decentralizednonconv, lian2017can,li2019communication,tang2018communication,koloskova2019decentralizedcon,reisizadeh2019robust} have been proposed for optimizing Eq.~(\ref{lossfunction}). For example, \cite{lian2017can} develops the decentralized stochastic gradient descent (D-SGD) method for non-convex problems, showing that  the decentralized SGD has the same leading term  of the convergence rate with that of the centralized SGD, achieving the linear speedup  with respect to the number of workers.
On  the  other hand, in a decentralized training system, a major concern  is the communication overhead. A lot of efforts \cite{koloskova2019decentralizedcon,koloskova2019decentralizednonconv,tang2018communication,li2019communication} have been made to reduce it. For instance, \cite{li2019communication}  proposes a periodic decentralized SGD (PD-SGD) with multiple local updates. It is to perform communication at every $p$ ($p>1$) iterations to reduce the number of communication rounds.

However, all the aforementioned decentralized methods only focus on the conventional SGD method. In practice, momentum SGD  is more widely used in training deep learning models than that without momentum. The reason is that momentum SGD usually converges faster and generalizes better \cite{yu2019linear}. Thus, to bridge this gap, it is necessary to explore the decentralized SGD with momentum, as well as reducing the communication overhead.  To address this issue, we propose the \textit{periodic decentralized momentum SGD} (PD-SGDM). Since both D-SGD and PD-SGD can achieve linear speedup regarding the number of workers, a natural question is whether our PD-SGDM  can still achieve a linear speedup regarding the number of workers or not. In fact, it is not trivial to answer this question. In particular, the momentum, periodic communication, and the topology of the decentralized system mingle together so that it is difficult to figure out its convergence behavior theoretically. In this paper, our theoretical result in Corollary~\ref{cor_1} establishes the connection between the convergence rate and the aforementioned factors, disclosing under which condition our proposed PD-SGDM can achieve the linear speedup. To the best of our knowledge, this is the first work achieving this result. 

Furthermore, although our proposed PD-SGDM  can reduce the number of communication rounds by using periodic communication, yet its communication cost in each communication round still might be the bottleneck, degrading the practical performance. For example, GPT-2 \cite{radford2019language} has about 1.5 billion parameters so that the communication cost is huge in each communication round. To alleviate the communication cost, some works \cite{koloskova2019decentralizednonconv,koloskova2019decentralizedcon,tang2018communication} propose to compress the communicated parameters. For instance,  \cite{koloskova2019decentralizednonconv} proposed CHOCO-SGD, which compresses the communicated parameters with an arbitrary compression ratio to reduce the communication cost in each round. \cite{tang2018communication} developed DeepSqueeze, which also admits an arbitrary compression ratio when compressing the communicated parameters. 
However, these methods didn't consider momentum SGD as well as periodic communication. 
To address these issues, we further propose a\textit{ communication-efficient periodic decentralized momentum SGD} (CPD-SGDM) method, which employs a compression operator to compress the communicated parameters between different workers. As a result, the periodic and compressed communication strategy can reduce communication costs significantly. However, the introduced compressed communication strategy makes it more challenging to answer the aforementioned question due to the entanglement of different factors.  In Corollary~\ref{cor_2}, we theoretically disclose the effect of different factors in CPD-SGDM on the convergence rate. To the best of our knowledge, this is also the first work achieving this result.  Finally, we conduct extensive experiments to evaluate the performance of our proposed two methods. Both of them show superior performance over existing methods. In light of the above, we summarize the contributions of our work as follows:
\begin{itemize}
	\item We propose a periodic decentralized momentum SGD (PD-SGDM) method. Our theoretical result provides the conditions under which PD-SGDM enjoys a linear speedup regarding the number of workers. 
	\item We also propose a  communication-efficient periodic decentralized momentum SGD (CPD-SGDM) method to further reduce the communication cost. Our theoretical result also discloses how the periodic communication, compressed communication, and topology affect the linear speedup regarding the number of workers. 
	\item Experiments on different learning tasks are conducted to verify the performance of our proposed methods.  Our new algorithms consistently outperform the existing counterpart methods. 
\end{itemize}

\section{Related Works}

\textbf{Decentralized Training}. In the past few years, decentralized training has attracted a lot of attention due to its superior communication efficiency over the centralized SGD. Specifically, in a centralized training system, there is a central server and multiple workers. Each worker computes gradient based on the local data and sends the gradient to the central server. As a result, the communication happening on the central server increases with the number of workers. Congestion will appear when there are a lot of workers.  In contrast to centralized training, there is no central server in a decentralized training system. Each worker only needs to communicate with its neighbors. Thus, its communication is efficient.  Due to its communication efficiency and the emergence of distributed data, numerous decentralized training methods \cite{lian2017can,koloskova2019decentralizedcon,koloskova2019decentralizednonconv,yu2019linear,li2019communication,alghunaim2019linearly,assran2018stochastic,tang2018communication} have been proposed.  In particular, \cite{lian2017can} studies the convergence rate of the decentralized SGD method, showing that it has a similar convergence rate as the centralized variant. From then on, different variants have been proposed for different settings. \cite{yu2019linear} studied the decentralized SGD with momentum. However, its communication overhead is even larger than the conventional decentralized SGD because it needs to communication both gradient and momentum. Additionally, it doesn't consider periodic communication to reduce the communication overhead. 
Except \cite{yu2019linear}, all existing decentralized methods ignore momentum SGD, which is more widely used in practical applications. 

\textbf{Periodic Communication}. With the development of deep neural networks, the model size is usually large. For example, ResNet50 \cite{he2016deep} has about 25 million parameters, and GPT-2 \cite{radford2019language} has about 1.5 billion parameters. Thus, when the worker communicates with its neighboring workers, the communication cost is high, slowing down the convergence speed. Periodic communication, which is first studied in the centralized training \cite{yu2019linear,yu2019parallel,stich2018local,lin2018don}, is introduced to reduce the  number of communication rounds. For example, \cite{li2019communication} proposed the periodic decentralized SGD method, which conducts multiple local updates and then performs communication to reduce the number of communication rounds. However, all these methods only focus on the conventional SGD method, and they cannot be directly applied to momentum SGD.

\textbf{Compressed Communication}. Compressing the communicated parameter is another widely used strategy to reduce the communication cost. Specifically, it employs the compression operator defined in Definition~\ref{compress_op}, such as sparsification \cite{lin2017deep,aji2017sparse} or quantization \cite{alistarh2017qsgd,wen2017terngrad}, to reduce the cost in each communication round. It is also first studied in  centralized training \cite{alistarh2017qsgd,wen2017terngrad,bernstein2018signsgd,stich2018sparsified} and afterwards introduced to  decentralized training.  For example, \cite{koloskova2019decentralizedcon} proposed CHOCO-SGD with compressed communication for the convex problem. \cite{koloskova2019decentralizednonconv}  applied CHOCO-SGD for the non-convex problem. \cite{tang2018communication} developed DeepSqueeze for compression with an arbitrary compression ratio. However, all existing methods fail to consider the more challenging momentum SGD. 

\begin{definition} \label{compress_op}
	Given a compression operator $Q:  \Re^d \rightarrow \Re^d$, if there exists $0<\delta\leq 1$, such that
	\begin{equation}
	\|\mathbf{x}- Q(\mathbf{x})\|^2\leq (1-\delta) \|\mathbf{x}\|^2 \ ,
	\end{equation}
	$Q$ is  $\delta$-contraction. 
\end{definition}

\section{Preliminaries}

\subsection{Notations}
At first, we introduce the following notations for our proposed methods.
$\mathbf{x}_t^{(k)}$ represents the local model parameter in the $k$-th worker  at the $t$-th iteration. 
$\bar{\mathbf{x}}_t = \frac{1}{K}\sum_{k=1}^{K}\mathbf{x}_t^{(k)}$ stands for the global averaged model parameter at the $t$-th iteration.
$\nabla F^{(k)}(\mathbf{x}_t^{(k)}; \xi_t^{(k)} ) $ denotes the local stochastic gradient of the $k$-th worker at the  $t$-th iteration.
$\nabla f^{(k)}(\mathbf{x}_t^{(k)}) = \mathbb{E}_{\xi_t^{(k)}\sim\mathcal{D}^{(k)}} \nabla F^{(k)}(\mathbf{x}_t^{(k)}; \xi_t^{(k)}) $ is the local full gradient of  the $k$-th worker node at the  $t$-th iteration.
$\nabla f(\mathbf{x}_t)=\frac{1}{K}\sum_{k=1}^{K}\nabla f^{(k)}(\mathbf{x}_t^{(k)}) $ is the global full gradient at the  $t$-th iteration.
$\mathbf{m}_t^{(k)}$ denotes the momentum of the $k$-th worker at the $t$-th iteration. 
$\mathbf{m}_t=\frac{1}{K}\sum_{k=1}^{K}\mathbf{m}_t^{(k)}$ denotes the averaged momentum  at the $t$-th iteration. 
$f^*$ is the minimum value of the loss function in Eq.~(\ref{lossfunction}).
$p$ represents  the communication period.

\subsection{Decentralized Training}
In a decentralized training system,  workers actually compose an undirected graph $\mathcal{G}=(V, W)$ where $V=[K]$ represents the worker set and the nonnegative matrix $W=[w_{ij}]\in \Re^{K\times K}$ describes the connection between different workers. In particular, a nonzero $w_{ij}$ indicates that there exists a connection between the $i$-th worker and the $j$-th worker.  As for matrix $W$, it is commonly assumed to be a stochastic matrix, which is shown as follows:

\begin{assumption} \label{graph}
	$W$ is a symmetric and double stochastic matrix with $0\leq w_{ij}\leq 1$, i.e.  it satisfies $W^T=W$, $W\mathbf{1}=\mathbf{1}$, and $\mathbf{1}^TW=\mathbf{1}^T$. 
\end{assumption}
Given Assumption~\ref{graph}, assume $W$ have  eigenvalues $|\lambda_n|\leq \cdots \leq |\lambda_2|\leq |\lambda_1|$, then  $\lambda_1=1$ and its spectral gap $\rho=1-|\lambda_2| \in (0, 1]$. Furthermore,  we have the following lemma. 

\begin{lemma} \label{lemma_w_spectral}\cite{koloskova2019decentralizedcon}
	For the doubly stochastic matrix $W$ defined in Definition~\ref{graph}, we  have 
	\begin{equation}
	\|W-\frac{1}{K}\mathbf{1}\mathbf{1}^T\|_2\leq 1-\rho
	\end{equation}
	where $1-\rho=|\lambda_2|<1$. $\lambda_2$ is the second largest eigenvalue of $W$.
\end{lemma}

Based on the aforementioned undirected graph, the workers collaboratively optimize Eq.~(\ref{lossfunction}). Specifically, for the decentralized optimization method, there are two steps in each iteration: update local model parameter and communicate with neighboring workers, which are defined as follows:
\begin{equation} \label{dec_sgd}
\mathbf{x}_{t+\frac{1}{2}}^{(k)} = \mathbf{x}_{t}^{(k)} - \eta \Delta_t^{(k)}  \ , \quad
\mathbf{x}_{t+1}^{(k)}  = \sum_{j\in \mathcal{N}_k} w_{kj}\mathbf{x}_{t+\frac{1}{2}}^{(j)}  \ ,
\end{equation}
where $\mathbf{x}_{t+\frac{1}{2}}^{(k)}$ represents the intermediate model parameter, $\Delta_t^{(k)}$ denotes the stochastic gradient $\nabla F^{(k)}(\mathbf{x}_{t}^{(k)} ; \xi)$ for decentralized SGD and  $\mathbf{m}_t^{(k)}$  for  decentralized  momentum SGD, $\mathcal{N}_k$ represents the neighboring workers of the $k$-th worker. In  this work, we will study how to apply periodic and  compressed communication to decentralize momentum SGD with theoretical convergence guarantee.

\subsection{Assumption}
As for the loss function, we have the following assumptions which are commonly used in studying the convergence of decentralized training methods for non-convex problems \cite{lian2017can,koloskova2019decentralizednonconv,li2019communication}.
\begin{assumption}\label{smooth}
	For each loss function on worker nodes, it is $L$-smooth, i.e.,
	\begin{equation}
	\|\nabla f^{(k)}(\mathbf{x}) - \nabla f^{(k)}(\mathbf{y})\| \leq L\|\mathbf{x}-\mathbf{y}\| , \quad \forall k\in [K], \forall \mathbf{x}\in \Re^{d}, \forall \mathbf{y}\in \Re^{d} \ .
	\end{equation}
\end{assumption}

\begin{assumption}\label{var1}
	For each loss function on worker nodes, we have $\sigma >0$ such that
	\begin{equation}
	\mathbb{E} [\|\nabla F^{(k)}(\mathbf{x}; \xi ) - \nabla f^{(k)}(\mathbf{x})\|^2] \leq \sigma^2, \quad \forall k\in [K], \forall \mathbf{x} \in \Re^{d} \ .
	\end{equation}
	
\end{assumption}

\begin{assumption} \label{norm}  
	For each loss function on  worker nodes, we  have $G>0$ such that
	\begin{equation}
	\|\nabla F^{(k)}(\mathbf{x}; \xi )\|^2 \leq G, \quad \forall \mathbf{x} \in \Re^{d} \ .
	\end{equation}
\end{assumption}

\section{Methodology}
\subsection{Periodic Decentralized Momentum SGD}

Momentum SGD is widely used in training deep neural networks since it converges faster and generalizes better than conventional SGD. However,  existing decentralized training methods \cite{lian2017can,koloskova2019decentralizedcon,koloskova2019decentralizednonconv} only focus on parallelizing conventional SGD and performing communication at each iteration. In addition, it is still unclear for the convergence behaviour of the periodic decentralized momentum SGD. 
In Algorithm~\ref{alg_dec_sgdm}, we propose the periodic decentralized momentum SGD. As the prototype of decentralized training in Eq.~(\ref{dec_sgd}), our algorithm also includes local computation within each worker and communication across different workers. Different from decentralized SGD, our algorithm updates the local model parameter with momentum SGD for multiple times as follows:
\begin{equation}
\begin{aligned}
& \mathbf{m}_{t}^{(k)} = \mu\mathbf{m}_{t-1}^{(k)} + \nabla F(\mathbf{x}_{t}^{(k)}; \xi_{t}^{(k)}) \ ,\\
& \mathbf{x}_{t+\frac{1}{2}}^{(k)}=\mathbf{x}_{t}^{(k)} - \eta\mathbf{m}_t^{(k)} \ ,
\end{aligned}
\end{equation} 
where  $\mathbf{x}_{t+\frac{1}{2}}^{(k)}$ is the intermediate model parameter, $\eta>0$ is the step  size, and $\mu>0$ denotes the momentum  coefficient. Here, since we employ the periodic communication, the local computation  will be repeated for $p$ ($p>1$) times by directly setting $\mathbf{x}_{t+1}^{(k)} =\mathbf{x}_{t+\frac{1}{2}}^{(k)}$. After $p$ iterations (when mod($t+1$, $p$)=0), each worker communicates the intermediate model parameter with its neighboring workers to get the new model parameter $\mathbf{x}_{t+1}^{(k)} $, just as shown in Line 6 of Algorithm~\ref{alg_dec_sgdm}. Here, we  can see that our algorithm do not conduct communication at every iterations like Eq.~(\ref{dec_sgd}). Instead, it performs communication at every  $p$ iterations. Thus, it can save a  lot of communication cost. 

\begin{algorithm}[]
	\caption{Periodic Decentralized Momentum SGD (PD-SGDM)}
	\label{alg_dec_sgdm}
	\begin{algorithmic}[1]
		\REQUIRE $\mathbf{x}_{0}^{(k)}=\mathbf{x}_{0}$, $\mathbf{m}_{t}^{(k)} = \mathbf{0}$, $p>1$, $\eta>0$, $\mu>0$, $W$. Conduct following steps for all workers.
		\FOR{$t=0,\cdots, T-1$} 
		\STATE Compute gradient $\nabla F(\mathbf{x}_{t}^{(k)}; \xi_{t}^{(k)})$
		\STATE $\mathbf{m}_{t}^{(k)} = \mu\mathbf{m}_{t-1}^{(k)} + \nabla F(\mathbf{x}_{t}^{(k)}; \xi_{t}^{(k)})$
		\STATE $\mathbf{x}_{t+\frac{1}{2}}^{(k)}=\mathbf{x}_{t}^{(k)} - \eta\mathbf{m}_t^{(k)}$
		
		\IF {mod($t+1$, $p$)=0}
		\STATE $\mathbf{x}_{t+1}^{(k)} =\sum_{j\in \mathcal{N}_k}w_{kj}\mathbf{x}_{t+\frac{1}{2}}^{(j)}$
		\ELSE
		\STATE $\mathbf{x}_{t+1}^{(k)} =\mathbf{x}_{t+\frac{1}{2}}^{(k)}$
		\ENDIF
		
		\ENDFOR
	\end{algorithmic}
\end{algorithm}

It is worth noting that the momentum term and periodic communication introduce new challenges to the convergence analysis. As a result, the existing analysis regarding decentralized SGD cannot be applied to our algorithm, because their techniques are not sufficient to deal with the momentum term.  In the following, we  establish the convergence rate of the periodic decentralized momentum SGD. To the best of our knowledge, this is the first work achieving this result.

\begin{theorem} \label{theorem1}
	Under Assumption~\ref{graph}--\ref{norm}, if we choose $\eta<\frac{(1-\mu)^2}{2L}$ and $0<\mu<1$, we have
	\begin{equation}
	\begin{aligned}
	& \frac{1}{T}\sum_{t=0}^{T-1} \| \nabla   f(\bar{\mathbf{x}}_{t})\|^2  \leq \frac{2(1-\mu)(f(\mathbf{x}_{0})  -  f^* )}{\eta T}+ \frac{\mu\eta \sigma^2L}{(1-\mu)^2K}  \\
	& \quad \quad \quad \quad  \quad \quad \quad \quad  + \frac{\eta\sigma^2L}{(1-\mu)K} +  \frac{2\eta^2 p^2G^2L^2}{(1-\mu)^2} (1+\frac{4}{\rho^2})  \ . 
	\end{aligned}
	\end{equation}
\end{theorem}
The proof can be found in Supplement \ref{proof_theorem_1}. From Theorem~\ref{theorem1}, it can be seen that the last term is affected by the topology spectral gap $\rho$ and the communication period $p$. In particular, we have the following corollary. 

\begin{corollary} \label{cor_1}
	Under Assumption~\ref{graph}--\ref{norm}, if we choose $\eta=O(\frac{K^{1/2}}{T^{1/2}})$ and $p=O(\frac{T^{1/4}}{K^{\tau}})$ where $\tau>0$, we can get
	\begin{equation}
	\frac{1}{T}\sum_{t=0}^{T-1} \| \nabla   f(\bar{\mathbf{x}}_{t})\|^2  = O(\frac{1}{\sqrt{KT}}) + O(\frac{1}{\rho^2K^{2\tau-1}\sqrt{T}} ) \ .
	\end{equation}
\end{corollary}
\begin{remark}
	When $\tau\leq \frac{3}{4}$, the second term dominates the first term. Then, the  spectral gap  $\rho$ will slow down the convergence speed. When $\tau > \frac{3}{4}$, the first term dominates the second term. The spectral gap only affects the high  order term. Thus, the convergence speed is $O(\frac{1}{\sqrt{KT}})$, which indicates a linear speedup with respect to the number of workers. 
\end{remark}

\subsection{Communication-Efficient Periodic Decentralized Momentum SGD}

With the development of deep neural networks in recent years, the model size is usually large, which will cause large communication overhead and degrade the performance for Algorithm~\ref{alg_dec_sgdm}.  Thus, it is necessary to reduce the communication cost further. Here, we propose the communication-efficient periodic decentralized momentum SGD in Algorithm~\ref{alg_dec_sgdm_com}. In detail, we employ the \textit{compressed operator} defined in  Definition~\ref{compress_op} to compress the communicated parameter between different workers to reduce the cost in each communication round. Specifically, Algorithm~\ref{alg_dec_sgdm_com} still uses momentum SGD to update its  intermediate  model parameter as Algorithm~\ref{alg_dec_sgdm}, but it has a different communication strategy. More specifically, each  worker $k$ stores an auxiliary model parameter $\hat{\mathbf{x}}_{t}^{(j)} \in \Re^{d}$ for each neighboring worker  $j\in\mathcal{N}_k$.  When mod($t+1$, $p$)=0,  each worker updates its local model parameter $\mathbf{x}_{t+1}^{(k)}$ based on its  intermediate model parameter $\mathbf{x}_{t+\frac{1}{2}}^{(k)}$  and the stored auxiliary model parameter $\hat{\mathbf{x}}_{t}^{(j)}$ as follows:
\begin{equation}
\begin{aligned}
& \mathbf{x}_{t+1}^{(k)} =\mathbf{x}_{t+\frac{1}{2}}^{(k)} + \gamma\sum_{j\in \mathcal{N}_k}w_{kj}(\hat{\mathbf{x}}_{t}^{(j)} - \hat{\mathbf{x}}_{t}^{(k)}) \ ,
\end{aligned}
\end{equation}
where $\gamma>0$ can be viewed as the consensus step size. Here, the auxiliary model parameter $\hat{\mathbf{x}}_{t}^{(j)}$ is introduced  to admit high compression ratio, otherwise it will not converge \cite{koloskova2019decentralizedcon}.
Then,  each worker communicates the following  compressed variable $\mathbf{q}_{t}^{(k)} $  with its neighboring workers:
\begin{equation}
\begin{aligned}
\mathbf{q}_{t}^{(k)} = Q(\mathbf{x}_{t+1}^{(k)}  - \hat{\mathbf{x}}_{t}^{(k)}) \ .
\end{aligned}
\end{equation}
After communication, each worker updates its auxiliary model parameter as  follows:
\begin{equation}
\hat{\mathbf{x}}_{t+1}^{(j)} = \hat{\mathbf{x}}_{t}^{(j)} +\mathbf{q}_{t}^{(j)}  \ .
\end{equation}
Intuitively, $\mathbf{q}_{t}^{(k)}$ can be viewed as the error compensation for compression so that we can use the compression operator with a high compression ratio, such as the sigh operator \cite{bernstein2018signsgd}. When $\text{mod}(t+1, p)\neq0$, similar with Algorithm~\ref{alg_dec_sgdm}, there is no communication, and the model parameters are updated as shown in Line 11-12 in Algorithm~\ref{alg_dec_sgdm_com}.

\begin{algorithm}[]
	\caption{Communication-Efficient Periodic Decentralized Momentum SGD (CPD-SGDM)}
	\label{alg_dec_sgdm_com}
	\begin{algorithmic}[1]
		\REQUIRE $\mathbf{x}_{0}^{(k)}=\mathbf{x}_{0}$, $\mathbf{m}_{t}^{(k)} = \mathbf{0}$, $p>1$, $\eta>0$, $\mu>0$, $\gamma>0$, $W$. Conduct following steps for all workers.
		\FOR{$t=0,\cdots, T-1$} 
		\STATE Compute gradient $\nabla F(\mathbf{x}_{t}^{(k)}; \xi_{t}^{(k)})$
		\STATE $\mathbf{m}_{t}^{(k)} = \mu\mathbf{m}_{t-1}^{(k)} + \nabla F(\mathbf{x}_{t}^{(k)}; \xi_{t}^{(k)})$
		\STATE $\mathbf{x}_{t+\frac{1}{2}}^{(k)}=\mathbf{x}_{t}^{(k)} - \eta\mathbf{m}_t^{(k)}$
		
		\IF {mod($t+1$, $p$)=0}
		\STATE $\mathbf{x}_{t+1}^{(k)} =\mathbf{x}_{t+\frac{1}{2}}^{(k)} + \gamma\sum_{j\in \mathcal{N}_k}w_{kj}(\hat{\mathbf{x}}_{t}^{(j)} - \hat{\mathbf{x}}_{t}^{(k)})$
		\STATE $\mathbf{q}_{t}^{(k)} = Q(\mathbf{x}_{t+1}^{(k)}  - \hat{\mathbf{x}}_{t}^{(k)})$
		\STATE Send $\mathbf{q}_{t}^{(k)} $ and receive $\mathbf{q}_{t}^{(j)}$ for $j\in  \mathcal{N}_k$
		\STATE $\hat{\mathbf{x}}_{t+1}^{(j)} = \hat{\mathbf{x}}_{t}^{(j)} +\mathbf{q}_{t}^{(j)}  $ for $j\in  \mathcal{N}_k$
		\ELSE
		\STATE $\mathbf{x}_{t+1}^{(k)} =\mathbf{x}_{t+\frac{1}{2}}^{(k)}$
		\STATE $\hat{\mathbf{x}}_{t+1}^{(j)} = \hat{\mathbf{x}}_{t}^{(j)} $ for $j\in  \mathcal{N}_k$
		\ENDIF
		
		\ENDFOR
	\end{algorithmic}
\end{algorithm}

The above communication protocol is first introduced by \cite{koloskova2019decentralizedcon}. But they only focus on the conventional decentralized SGD without momentum and periodic communication. Here, the momentum term, periodic communication, and compressed communication make it more challenging to study the convergence of Algorithm~\ref{alg_dec_sgdm_com}. In the following, we will show how these factors affect the convergence rate. 

\begin{theorem} \label{theorem2}
	Under Assumption~\ref{graph}--\ref{norm}, if we choose $\eta<\frac{(1-\mu)^2}{2L}$ and $0<\mu<1$, we have
	\begin{equation}
	\begin{aligned}
	& \frac{1}{T}\sum_{t=0}^{T-1} \| \nabla   f(\bar{\mathbf{x}}_{t})\|^2  \leq \frac{2(1-\mu)(f(\mathbf{x}_{0})  -  f^* )}{\eta T}+ \frac{\mu\eta \sigma^2L}{(1-\mu)^2K} \\
	& \quad \quad \quad \quad \quad \quad \quad \quad  + \frac{4\eta^2p^2G^2L^2}{(1-\mu)^2} (1+\frac{4}{\alpha^2})   + \frac{\eta\sigma^2L}{(1-\mu)K}  \ ,
	\end{aligned}
	\end{equation}
	where $\alpha=\frac{\rho^2\delta}{82}$.
\end{theorem}
The proof can be found in Supplement \ref{proof_theorem_2}.  Compared with  Algorithm~\ref{alg_dec_sgdm}, the convergence rate of Algorithm~\ref{alg_dec_sgdm_com} depends on both the spectral gap $\rho$ and the  compression parameter $\delta$. Moreover,  it has a worse dependence on the  spectral gap $\rho$ than that of  Algorithm~\ref{alg_dec_sgdm}.

\begin{corollary} \label{cor_2}
	Under Assumption~\ref{graph}--\ref{norm}, if we choose $\eta=O(\frac{K^{1/2}}{T^{1/2}})$ and $p=O(\frac{T^{1/4}}{K^{\tau}})$ where $\tau>0$, we can get
	\begin{equation}
	\frac{1}{T}\sum_{t=0}^{T-1} \| \nabla   f(\bar{\mathbf{x}}_{t})\|^2  = O(\frac{1}{\sqrt{KT}}) + O(\frac{1}{\rho^4\delta^2K^{2\tau-1}\sqrt{T}} ) \ .
	\end{equation}
\end{corollary}

\begin{remark}
	Similar  with Corollary~\ref{cor_1},  when $\tau\leq \frac{3}{4}$, the second term is  the dominant term so that the  spectral gap  $\rho$ and the compression parameter $\delta$  slow down the convergence speed. When $\tau > \frac{3}{4}$, the first term is dominant. The spectral gap and  the compression parameter  only affect the high  order term. Thus, the convergence speed is $O(\frac{1}{\sqrt{KT}})$, which indicates a linear speedup with respect to the number of workers. 
\end{remark}

\section{Experiments}
In this section, we will conduct extensive experiments to verify the performance of our proposed Algorithm~\ref{alg_dec_sgdm} and Algorithm~\ref{alg_dec_sgdm_com}.  

\subsection{Experimental Settings}
In our experiments, 8 workers connected in a ring topology are used and each worker is an  Nvidia Tesla P40 GPU. Thus, each worker only needs to communicate with  two  neighboring workers. Our algorithms are implemented with PyTorch \cite{paszke2019pytorch}. To evaluate the  performance of our algorithms, we use two datasets in our experiments. The detailed settings for these two datasets are described as follows.
\vspace{-5pt}
\begin{itemize}
	\item CIFAR-10 \cite{krizhevsky2009learning} is a widely used image dataset for the classification task. Here, the deep neural network used for this dataset is ResNet20 \cite{he2016deep}. Following \cite{koloskova2019decentralizednonconv}, we train this model for 300 epochs totally.  The initial learning rate is first set to 0.1 and then decayed by 0.1 at epoch 150 and 225. The weight decay coefficient is set to $1\times 10^{-4}$, and the momentum coefficient $\mu$ is set to 0.9. The mini-batch size on each worker is 16. 
	\item ImageNet \cite{ILSVRC15} is a benchmark dataset for the large-scale image classification task. In this experiment, we use ResNet50 \cite{he2016deep} for this dataset. The number of epochs is 90. The initial learning rate is also 0.1 and decayed by 0.1 at epoch 30, 60, and 80. The weight decay and the momentum coefficient are the same as those of CIFAR-10. The mini-batch size on each worker is 32. 
	
\end{itemize}
Moreover, for Algorithm~\ref{alg_dec_sgdm_com}, the  consensus step size $\gamma$ is set to 0.4 for CIFAR-10 and 0.5 for  ImageNet. To compress the communicated parameter, we  use the sign operator \cite{bernstein2018signsgd}.

\subsection{Experimental Results}
We first compare our PD-SGDM with the regular centralized momentum SGD (C-SGDM). Here, we use different communication periods $p=4,8,16$. In Figure~\ref{train_loss_sgd_iters_cifar10} and Figure~\ref{train_loss_sgd_iters_imagenet}, we report the training loss of ResNet20 and ResNet50 regarding the number of iterations respectively. From these two figures, we can see that  PD-SGDM with different $p$ and C-SGDM converge to almost the same value.   In  addition, from Figure~\ref{test_acc_sgd_epochs_cifar10} and Figure~\ref{test_acc_sgd_epochs_imagenet}, we can find that  PD-SGDM with different $p$ have almost the same final testing accuracy as C-SGDM. In other words, the periodic communication does not hurt the generalization performance of PD-SGDM.

\begin{figure}[!htbp]
	\centering 
	\subfigure[ResNet20@CIFAR-10]{
		\includegraphics[width=0.36\textwidth]{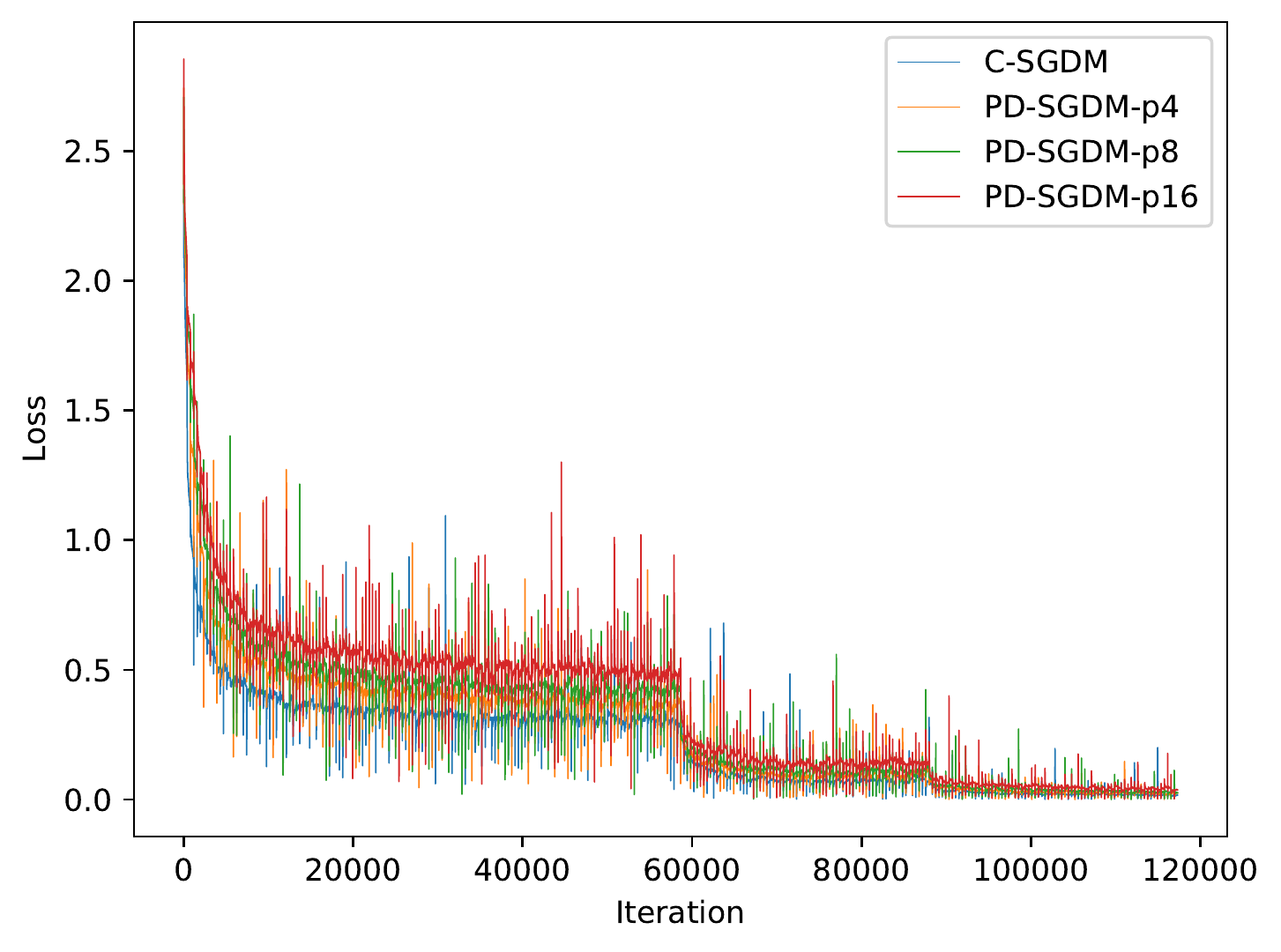}
		\label{train_loss_sgd_iters_cifar10}
	}
	\subfigure[ResNet50@ImageNet]{
		\includegraphics[width=0.3558\textwidth]{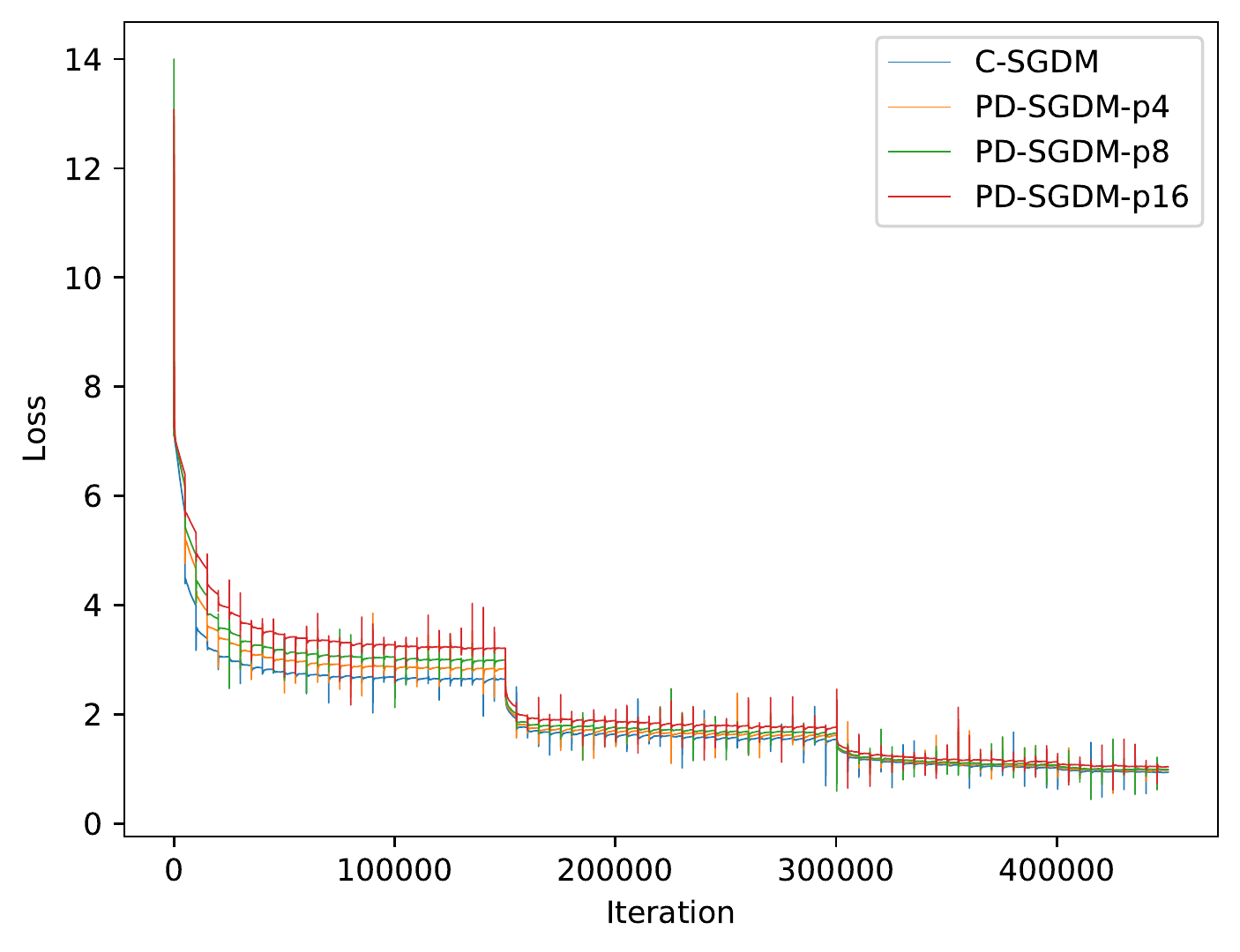}
		\label{train_loss_sgd_iters_imagenet}
	}
	\subfigure[ResNet20@CIFAR-10]{
		\includegraphics[width=0.3558\textwidth]{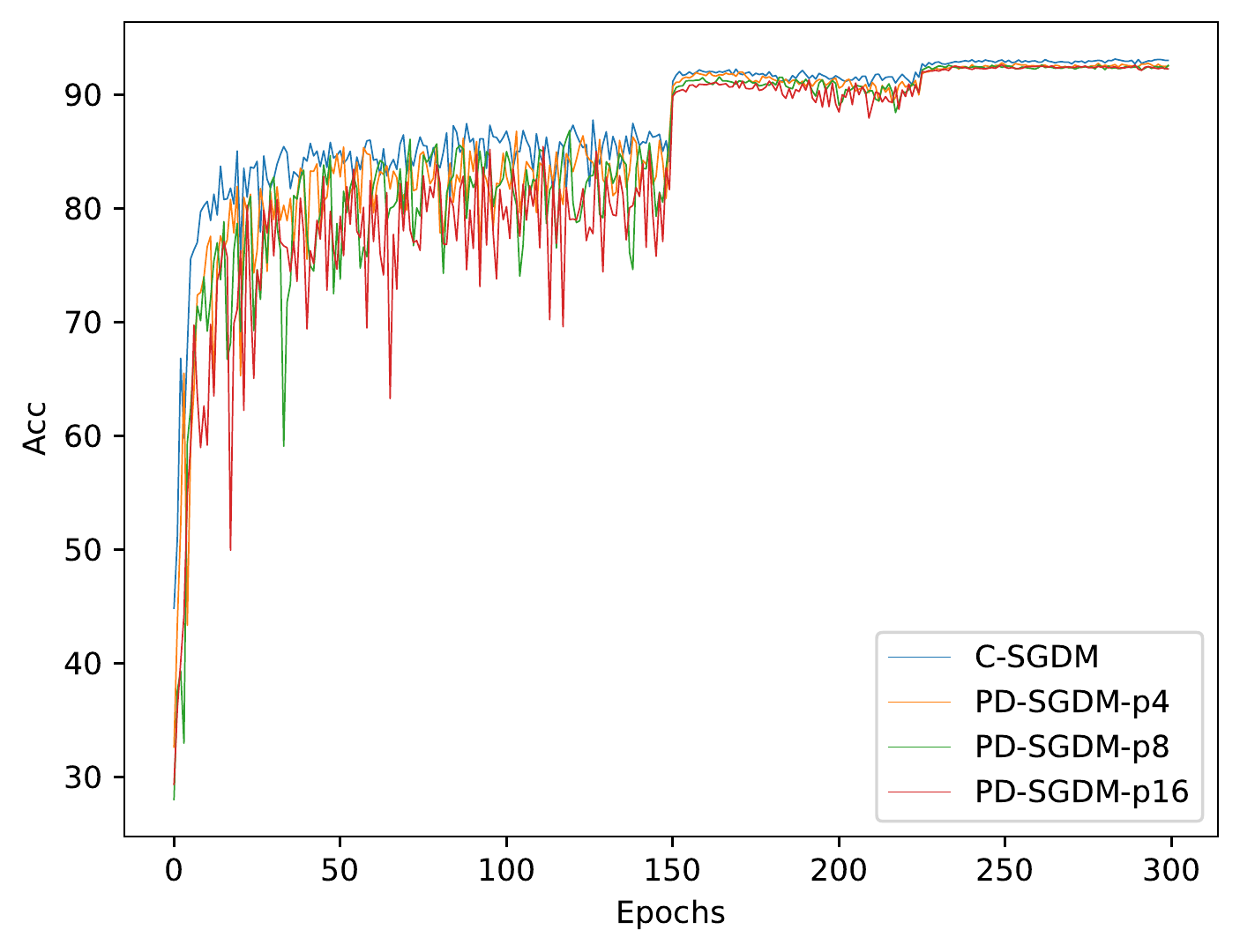}
		\label{test_acc_sgd_epochs_cifar10}
	}
	\subfigure[ResNet50@ImageNet]{
		\includegraphics[width=0.3558\textwidth]{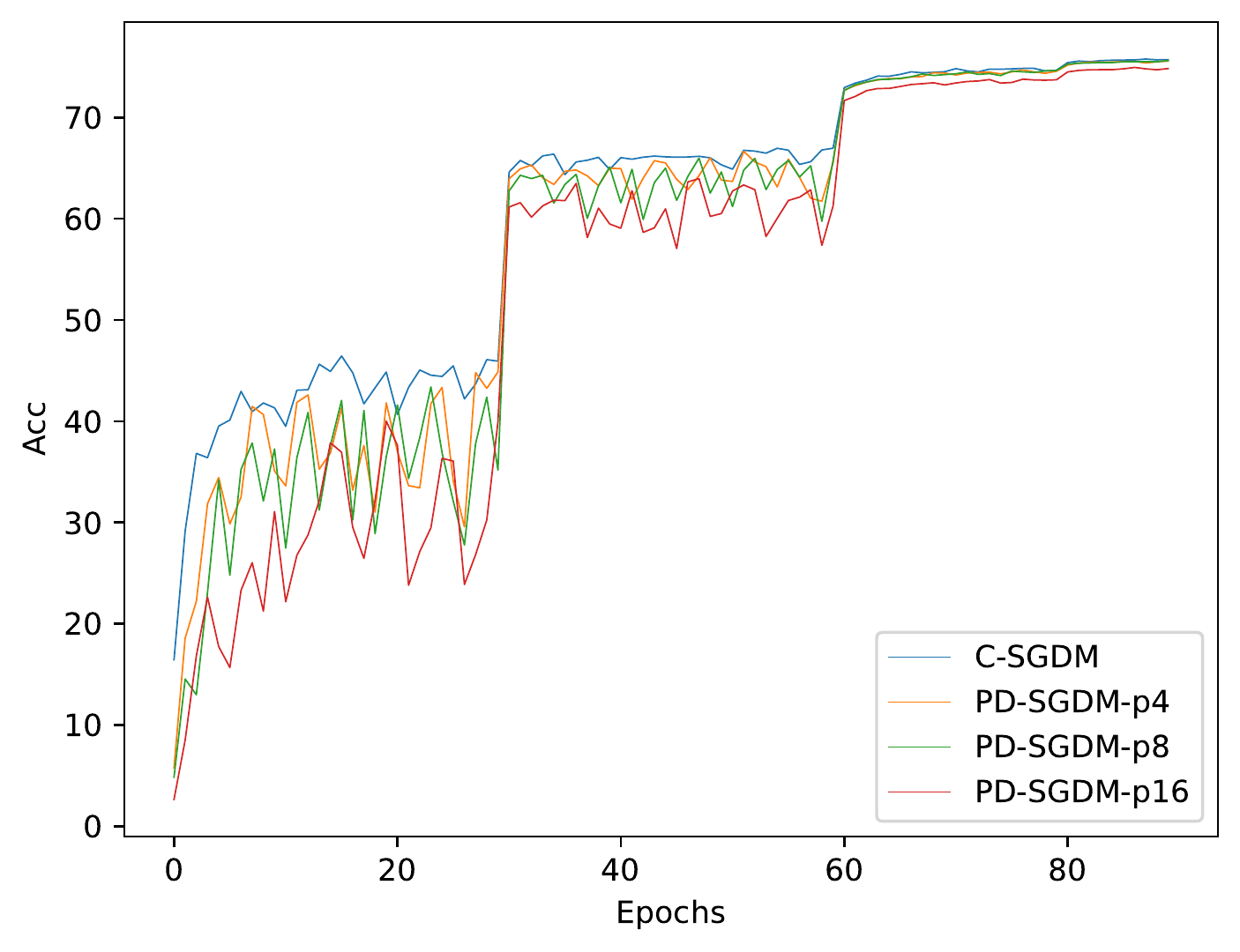}
		\label{test_acc_sgd_epochs_imagenet}
	}
	\caption{ The convergence and generalization performance of PD-SGDM. }
	\label{full_train_test}
\end{figure}

Furthermore, we report the communication cost of PD-SGDM in Figure~\ref{test_acc_sgd_bits_cifar10} and~\ref{test_acc_sgd_bits_imagenet}. In particular, we show the testing accuracy of PD-SGDM regarding the communication cost (MB) for CIFAR10 in Figure~\ref{test_acc_sgd_bits_cifar10} and the testing accuracy of PD-SGDM for ImageNet in Figure~\ref{test_acc_sgd_bits_imagenet}.  From these two figures, it can be seen that a larger $p$ leads to less communication cost but it does not hurt the convergence and generalization performance.

\begin{figure}[!htbp]
	\centering 
	\subfigure[ResNet20@CIFAR-10]{
		\includegraphics[width=0.3558\textwidth]{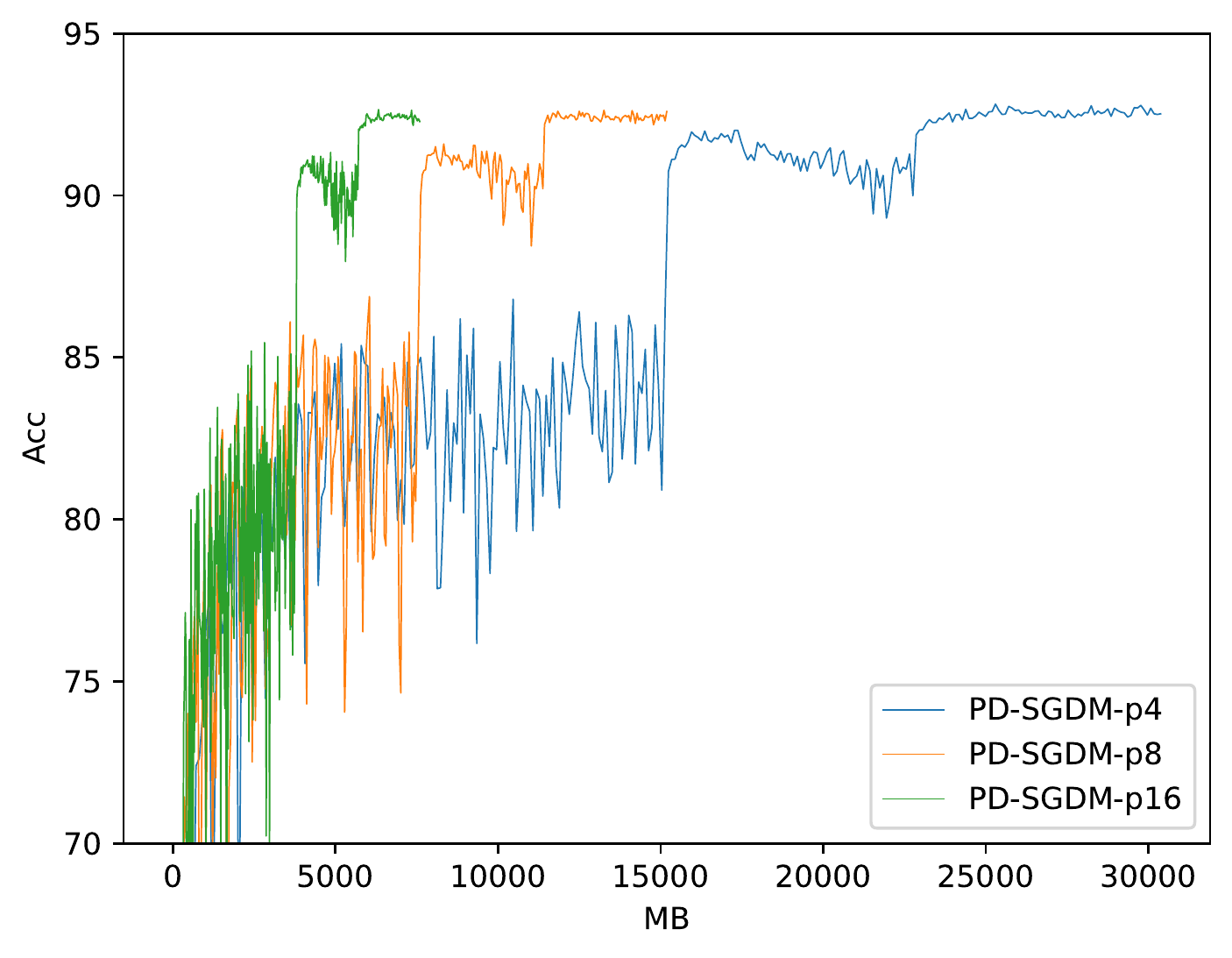}
		\label{test_acc_sgd_bits_cifar10}
	}
	\subfigure[ResNet50@ImageNet]{
		\includegraphics[width=0.3558\textwidth]{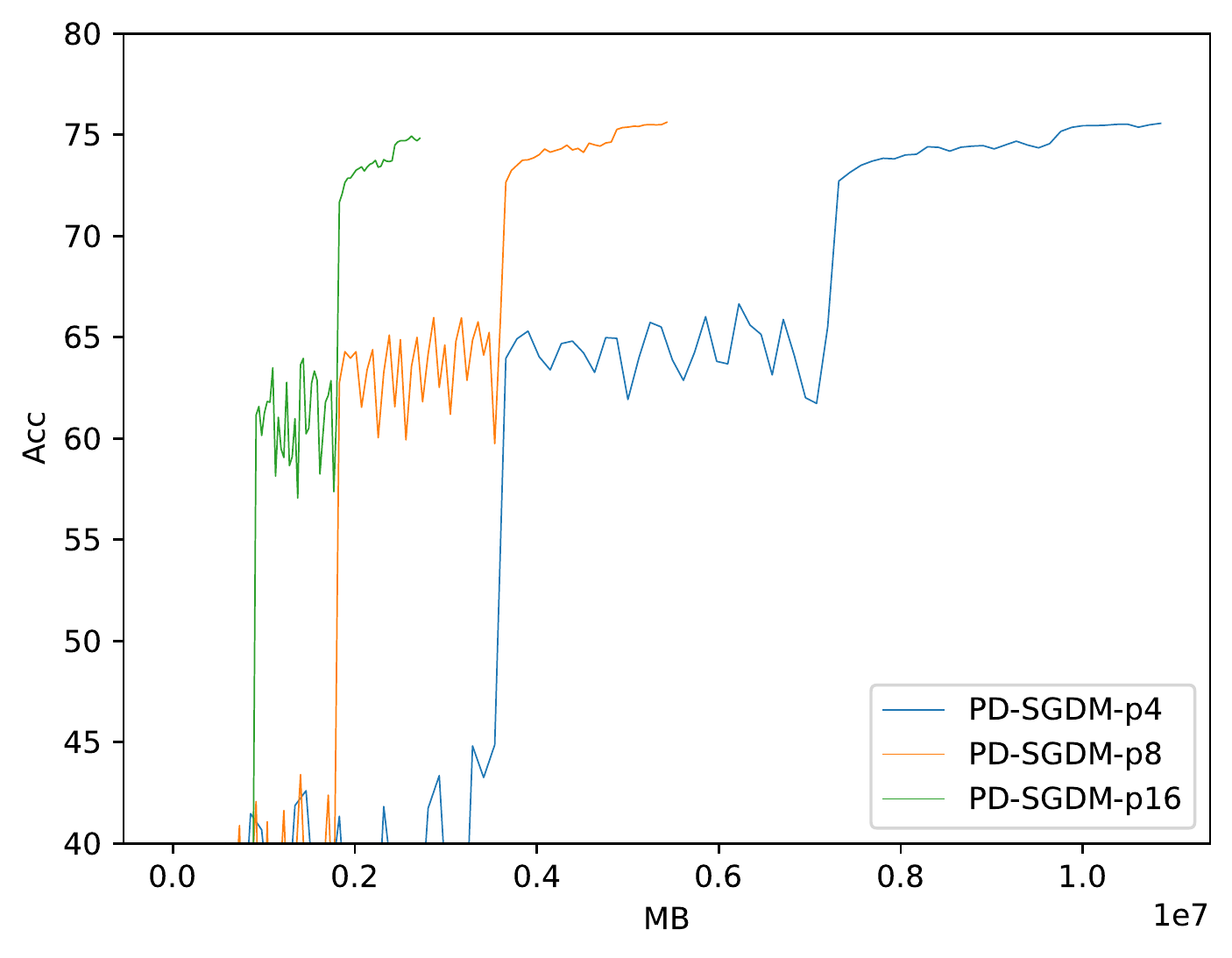}
		\label{test_acc_sgd_bits_imagenet}
	}
	\subfigure[ResNet20@CIFAR-10]{
		\includegraphics[width=0.3558\textwidth]{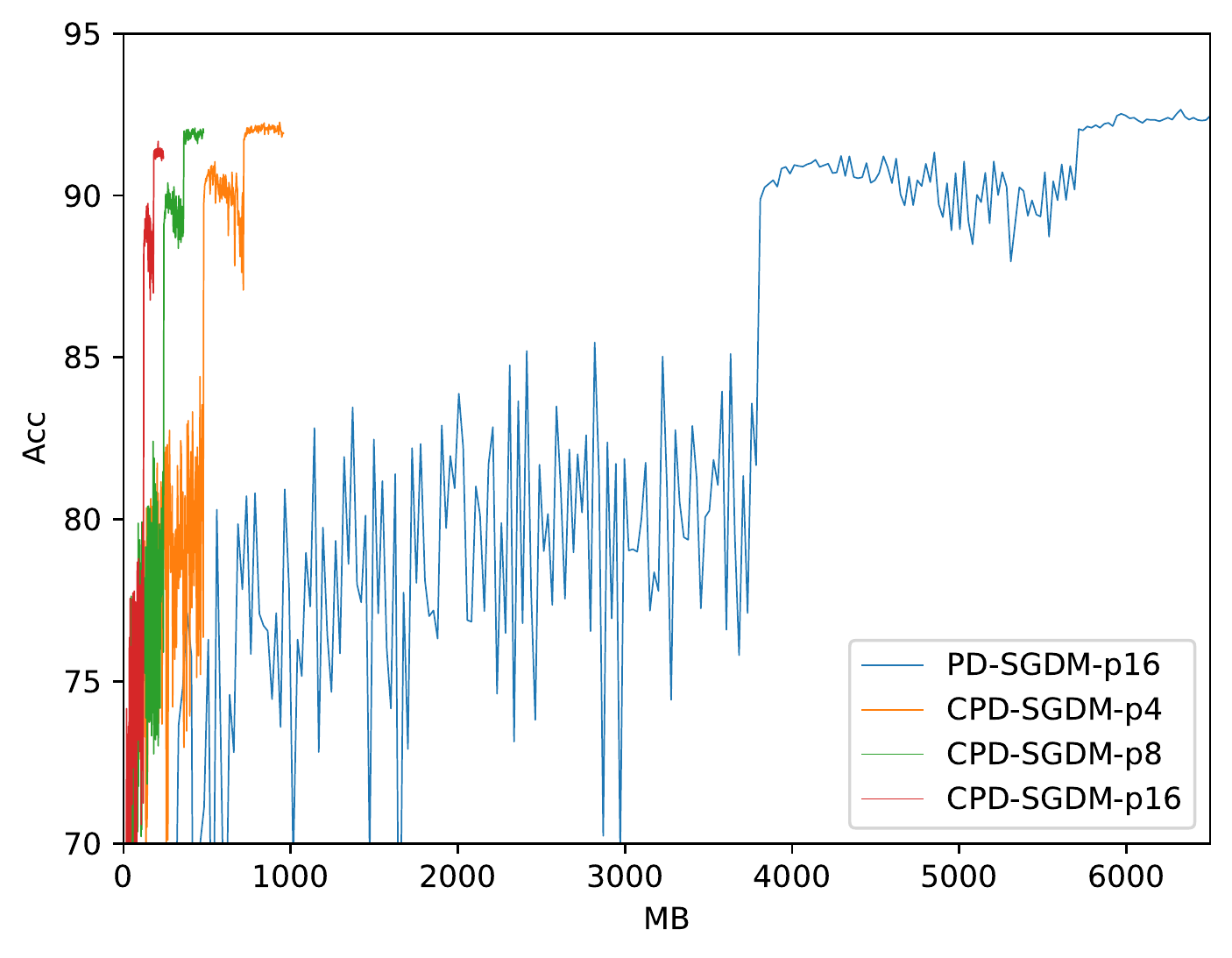}
		\label{test_acc_sgd_choco_bits_cifar10}
	}
	\subfigure[ResNet50@ImageNet]{
		\includegraphics[width=0.3558\textwidth]{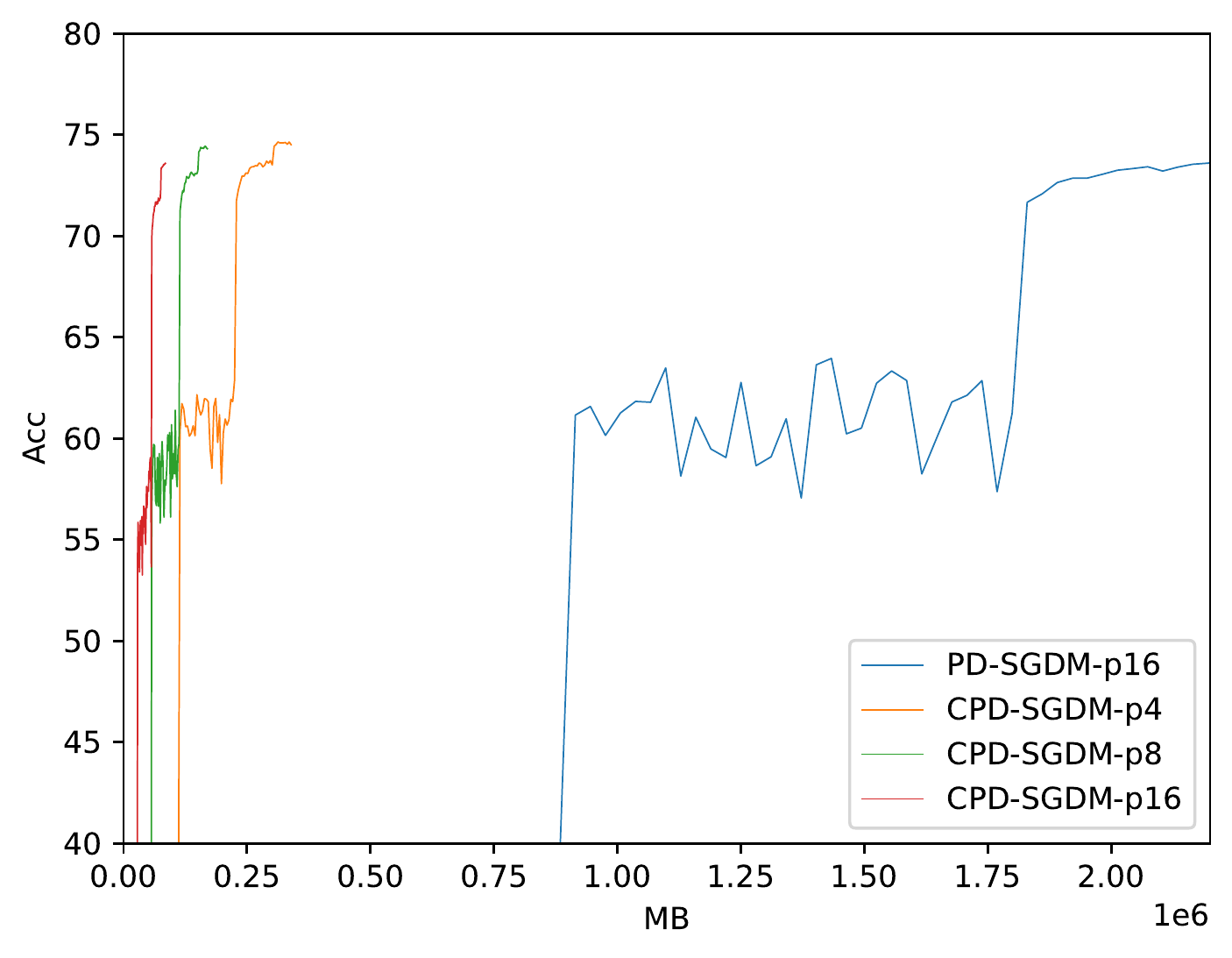}
		\label{test_acc_sgd_choco_bits_imagenet}
	}
	\caption{ The convergence and generalization performance of PD-SGDM (a-b) and CPD-SGDM (c-d) regarding the communication cost (MB).}
	\label{communication}
\end{figure}

Moreover, we conduct experiments to evaluate the performance of our Algorithm~\ref{alg_dec_sgdm_com}. In Figure~\ref{train_loss_sgd_choco_iters_cifar10} and Figure~\ref{train_loss_sgd_choco_iters_imagenet}, we compare the training loss of CPD-SGDM with that of the full-precision PD-SGDM (p=4) for  ResNet20 and ResNet50. From these  two figures, it is easy to find that CPD-SGDM converges to almost the same value with the full-precision PD-SGDM even though it  compresses the communicated parameters. In Figure~\ref{test_acc_sgd_choco_bits_cifar10} and Figure~\ref{test_acc_sgd_choco_bits_imagenet}, we demonstrate the testing accuracy of CPD-SGDM with different communication periods $p$ regarding the  number of epochs. It also can be found that CPD-SGDM has almost the same final testing performance with the full-precision PD-SGDM, which further  confirms the effectiveness of our Algorithm~\ref{alg_dec_sgdm_com}.

\begin{figure}[!htbp]
	\centering 
	\subfigure[ResNet20@CIFAR-10]{
		\includegraphics[width=0.363\textwidth]{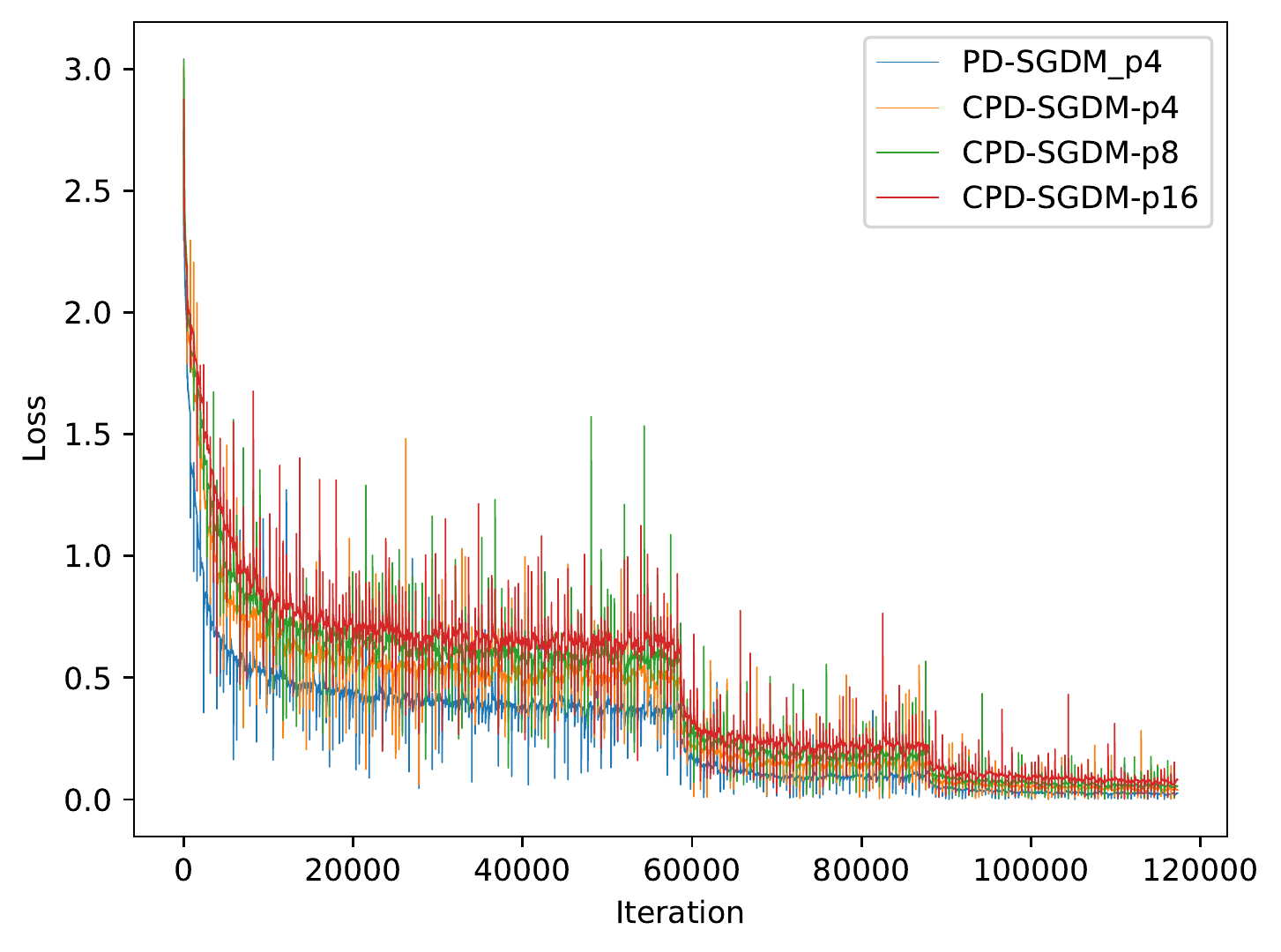}
		\label{train_loss_sgd_choco_iters_cifar10}
	}
	\subfigure[ResNet50@ImageNet]{
		\includegraphics[width=0.3558\textwidth]{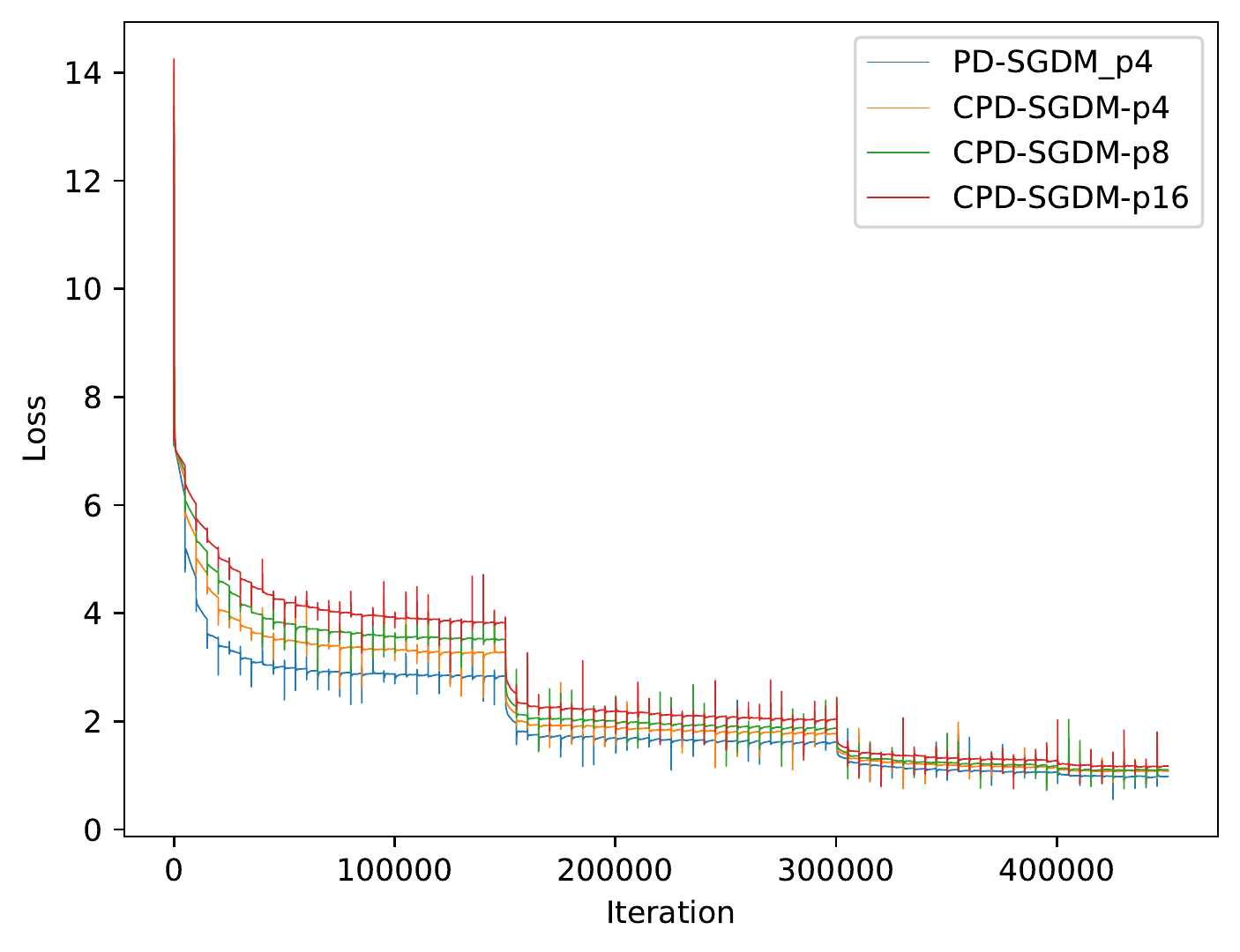}
		\label{train_loss_sgd_choco_iters_imagenet}
	}
	\subfigure[ResNet20@CIFAR-10]{
		\includegraphics[width=0.3558\textwidth]{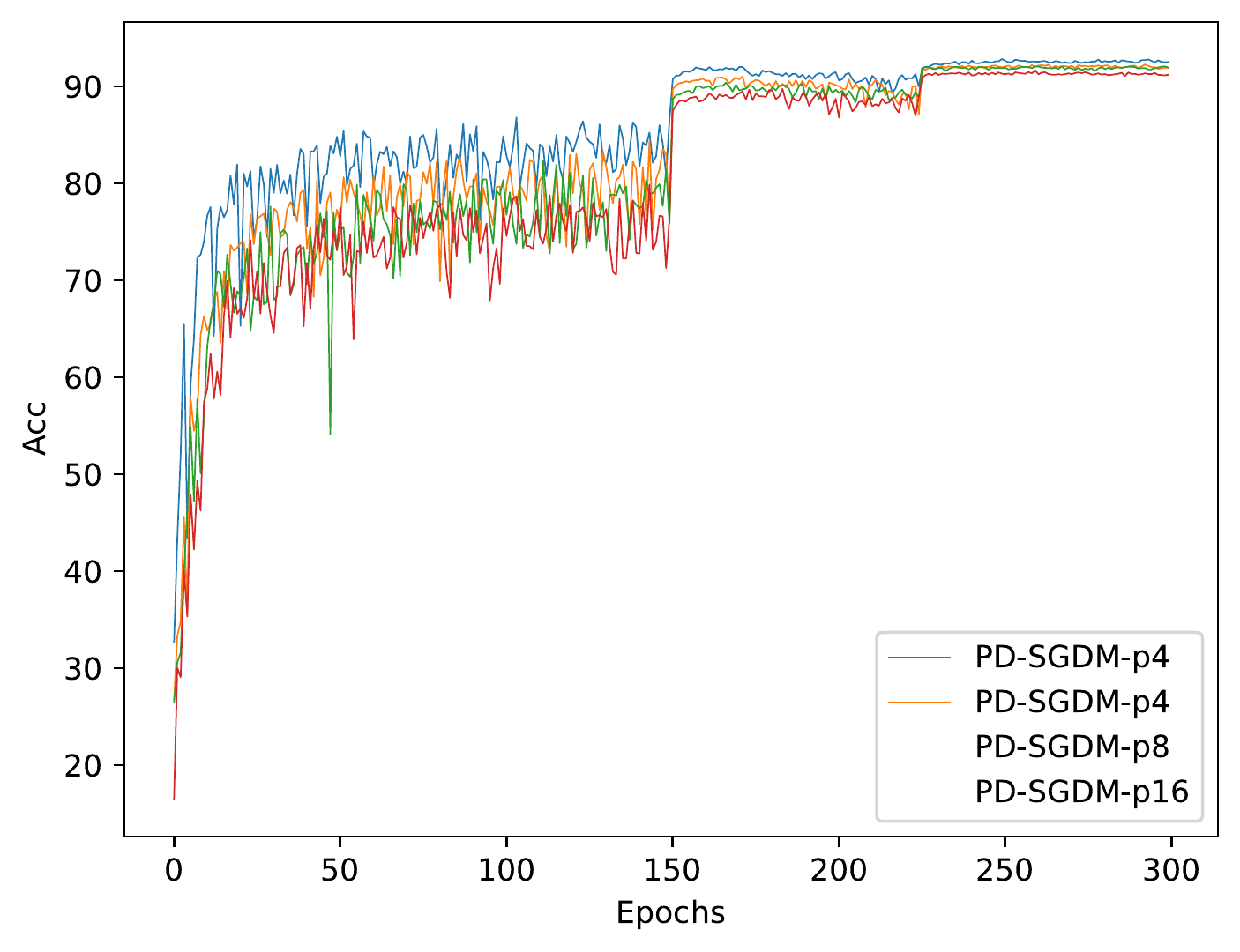}
		\label{test_acc_sgd_choco_epochs_cifar10}
	}
	\subfigure[ResNet50@ImageNet]{
		\includegraphics[width=0.3558\textwidth]{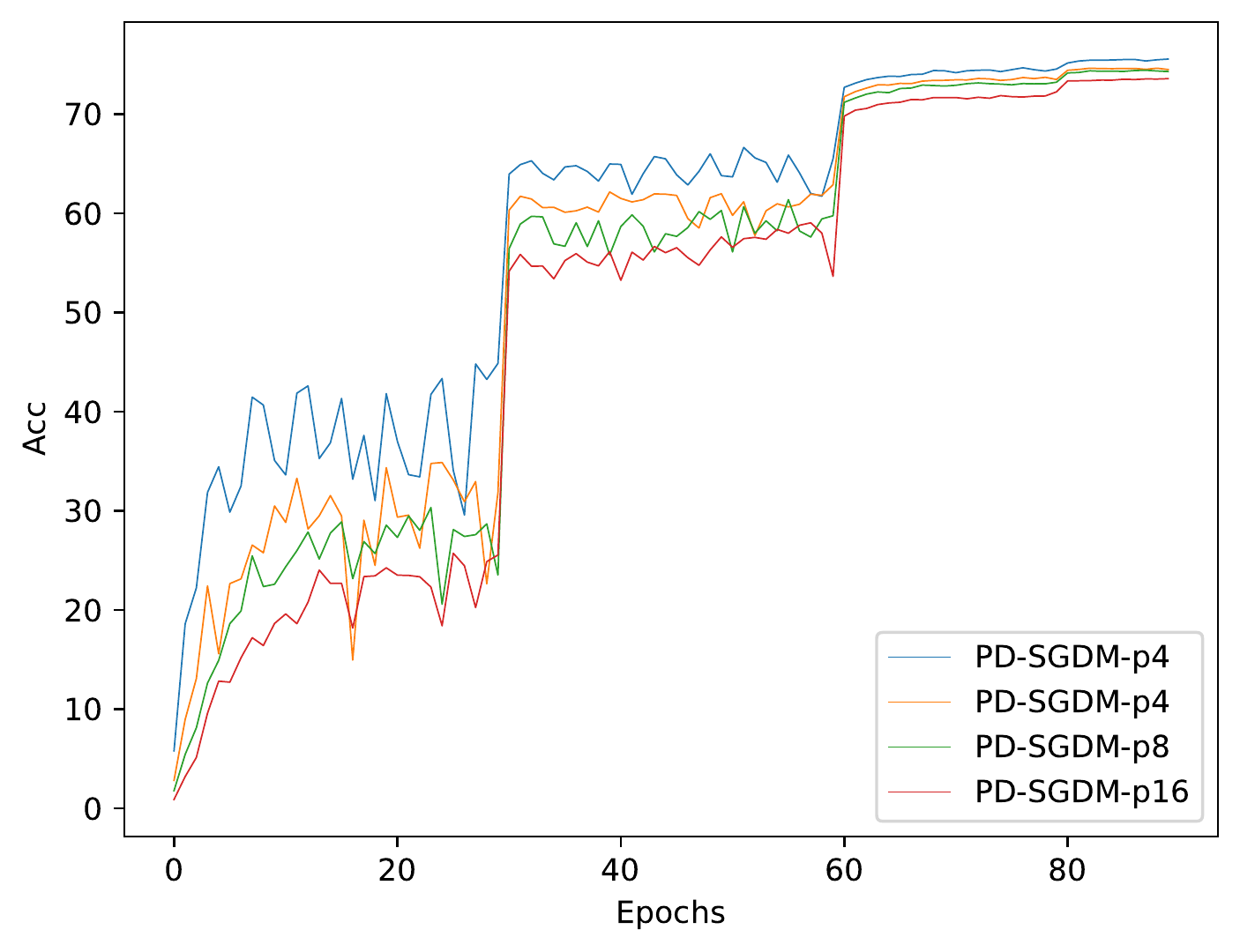}
		\label{test_acc_sgd_choco_epochs_imagenet}
	}
	\caption{ The convergence and generalization performance of CPD-SGDM. }
	\label{compress_train_test}
\end{figure}

At last, we  report the communication cost of CPD-SGDM in Figure~\ref{test_acc_sgd_choco_bits_cifar10} and~\ref{test_acc_sgd_choco_bits_imagenet}. Here, we show the testing accuracy regarding the communication cost for ResNet20 and ResNet50 respectively.  Obviously, we can find that CPD-SGDM has less communication cost than PD-SGDM (p=16) \footnote{Because PD-SGDM (p=16)  has much less communication cost than PD-SGDM with other communication periods (4, 8), we only compare CPD-SGDM with PD-SGDM (p=16).} because CPD-SGDM  has less communication cost in each communication round, which further confirms the communication efficiency of our proposed  CPD-SGDM .

\section{Conclusions}

Decentralized training methods have been actively studied in recent years. However, most of them only focus on the conventional SGD method, neglecting the more widely used momentum SGD method. In this paper, we bridged this gap and proposed a novel periodic decentralized momentum SGD method. It uses the momentum strategy to accelerate convergence and periodic communication approach to reduce the number of communication rounds. By carefully studying these factors, we disclose under which condition our method can achieve the linear speedup regarding the number of workers. Moreover, to reduce the communication cost, especially for large models, we further propose a  new communication-efficient periodic decentralized momentum SGD method. It is also proved to enjoy a linear speedup under certain conditions.  All these theoretical results are the first time to be disclosed.  At last, our empirical results have confirmed the correctness and effectiveness of our proposed two methods.

\bibliographystyle{plain} 
\bibliography{bibfile}

\onecolumn
\section{Supplement}


\subsection{Proof of  Theorem~\ref{theorem1}} \label{proof_theorem_1}
To prove the convergence, we use the following notation:
\begin{equation} \label{mat_1}
\begin{aligned}
& X_t = [\mathbf{x}_{t}^{(1)}, \mathbf{x}_{t}^{(2)}, \cdots, \mathbf{x}_{t}^{(K)}] \in \Re^{d\times K}  \ ,\\
& \bar{X}_t = [\bar{\mathbf{x}}_t, \bar{\mathbf{x}}_t, \cdots, \bar{\mathbf{x}}_t] \in \Re^{d\times K}  \ ,\\
& \Delta_t= [\mathbf{m}^{(1)}_t, \mathbf{m}^{(2)}_t , \cdots, \mathbf{m}^{(K)}_t] \in \Re^{d\times K}   \ .\\
\end{aligned}
\end{equation}

Then, we can represent Algorithm~\ref{alg_dec_sgdm}  as follows:
\begin{equation}
\begin{aligned}
&  X_{t+\frac{1}{2}}  = X_t - \eta \Delta_t  \ ,\\
& X_{t+1}= \begin{cases}
X_{t+\frac{1}{2}}  ,  & \mod($t+1$, $p$)\neq 0 \\
X_{t+\frac{1}{2}}W, & \mod($t+1$, $p$)=0\\
\end{cases} \ .
\end{aligned}
\end{equation}
Moreover, since $W$ is a doubly stochastic matrix, we have $W \frac{1}{K}\mathbf{1}\mathbf{1}^T =\frac{1}{K}\mathbf{1}\mathbf{1}^T$. Then, for $\forall t\geq 0$, we can get 
\begin{equation}
\bar{X}_{t+1} =  X_{t+1}\frac{1}{K}\mathbf{1}\mathbf{1}^T = (X_t - \eta\Delta_t ) \frac{1}{K}\mathbf{1}\mathbf{1}^T = \bar{X}_{t} - \eta \Delta_t \frac{1}{K}\mathbf{1}\mathbf{1}^T  \ .
\end{equation}
In other words,  for $\forall t\geq 0$,  we have
\begin{equation}
\bar{\mathbf{x}}_{t+1} = 	\bar{\mathbf{x}}_{t} - \frac{\eta }{K}\sum_{k=1}^{K}\mathbf{m}_t^{(k)} \ .
\end{equation}
Furthermore, we introduce an auxiliary variable for proving Theorem~\ref{theorem1} as follows:
\begin{equation}
\begin{aligned}
&\mathbf{z}_{t}= \begin{cases}
\bar{\mathbf{x}}_{t} ,  & t=0 \\
\frac{1}{1-\mu}\bar{\mathbf{x}}_{t}  - \frac{\mu}{1-\mu} \bar{\mathbf{x}}_{t-1} , & t\geq 1\\
\end{cases} \ .
\end{aligned}
\end{equation}

Then, it is easy to have the following equations \cite{yu2019linear}: 
\begin{equation} \label{aux_z}
\begin{aligned}
& \mathbf{z}_{t+1} - \mathbf{z}_{t} = -\frac{\eta}{1-\mu} \frac{1}{K}\sum_{k=1}^{K}\nabla F^{(k)}(\mathbf{x}_{t}^{(k)}; \xi_{t}^{(k)})  \ ,\\
& \mathbf{z}_{t} - \bar{\mathbf{x}}_{t}  = -\frac{\eta\mu}{1-\mu} \bar{\mathbf{m}}_{t} \ .
\end{aligned}
\end{equation}
where $\bar{\mathbf{m}}_{t} =\frac{1 }{K}\sum_{k=1}^{K}\mathbf{m}_t^{(k)} $.

\begin{lemma} \label{gradient_norm}
	Under Assumption~\ref{var1}, we have
	\begin{equation}
	\mathbb{E}[\| \frac{1}{K}\sum_{k=1}^{K}\nabla F(\mathbf{x}_{t}^{(k)}; \xi_{t}^{(k)})\|^2]  \leq \frac{\sigma^2}{K}+ \mathbb{E}[\|\frac{1}{K}\sum_{k=1}^{K}\nabla f^{(k)}(\mathbf{x}_{t}^{(k)})\|^2]   \ .
	\end{equation}
\end{lemma}
\begin{proof}
	\begin{equation}
	\begin{aligned}
	& \mathbb{E}[\| \frac{1}{K}\sum_{k=1}^{K}\nabla F(\mathbf{x}_{t}^{(k)}; \xi_{t}^{(k)})\|^2] \\
	& = \mathbb{E}[\| \frac{1}{K}\sum_{k=1}^{K}\nabla F(\mathbf{x}_{t}^{(k)}; \xi_{t}^{(k)}) - \frac{1}{K}\sum_{k=1}^{K} \nabla f^{(k)}(\mathbf{x}_{t}^{(k)}) + \frac{1}{K}\sum_{k=1}^{K}\nabla f^{(k)}(\mathbf{x}_{t}^{(k)})\|^2] \\
	& = \frac{1}{K^2}\sum_{k=1}^{K}\mathbb{E}[\| \nabla F(\mathbf{x}_{t}^{(k)}; \xi_{t}^{(k)}) - \nabla f^{(k)}(\mathbf{x}_{t}^{(k)})\|_2^2] + \mathbb{E}[\|\frac{1}{K}\sum_{k=1}^{K}\nabla f^{(k)}(\mathbf{x}_{t}^{(k)})\|^2] \\
	& \leq \frac{\sigma^2}{K}+ \mathbb{E}[\|\frac{1}{K}\sum_{k=1}^{K}\nabla f^{(k)}(\mathbf{x}_{t}^{(k)})\|^2]  \ .\\
	\end{aligned}
	\end{equation}
\end{proof}

\begin{lemma} \label{momentum_norm}
	Under Assumption~\ref{norm},  we have
	\begin{equation}
	\begin{aligned}
	& \mathbb{E} [\|\mathbf{m}_{t}^{(k)} \|^2]  \leq \frac{G^2}{(1-\mu)^2} \ .
	\end{aligned}
	\end{equation}
\end{lemma}

\begin{proof}
	\begin{equation}
	\begin{aligned}
	& \mathbb{E} [\|\mathbf{m}_{t}^{(k)} \|^2] =\mathbb{E} [\|\sum_{t'=0}^{t-1}\mu^{t-1-t'} \nabla F^{(k)}(\mathbf{x}_{t}^{(k)}; \xi_{t}^{(k)}) \|^2] \\
	& = \mathbb{E} [v_t^2\|\sum_{t'=0}^{t-1} \frac{\mu^{t-1-t'}}{v_t} \nabla F^{(k)}(\mathbf{x}_{t'}^{(k)}, \xi_{t'}^{(k)}) \|^2] \\
	& \leq \mathbb{E} [v_t\sum_{t'=0}^{t-1} {\mu^{t-1-t'}}\| \nabla F^{(k)}(\mathbf{x}_{t'}^{(k)}, \xi_{t'}^{(k)}) \|^2] \\
	& \leq \mathbb{E} [\frac{1}{1-\mu}\sum_{t'=0}^{t-1} {\mu^{t-1-t'}}\|\nabla F^{(k)}(\mathbf{x}_{t'}^{(k)}, \xi_{t'}^{(k)}) \|^2] \\
	& \leq \mathbb{E} [\frac{ G^2}{1-\mu}\sum_{t'=0}^{t-1} {\mu^{t-1-t'}}] \\
	& \leq \frac{ G^2}{(1-\mu)^2} \ ,
	\end{aligned}
	\end{equation}
	where $v_t=\sum_{t'=0}^{t-1} \mu^{t-1-t'}\leq \frac{1}{1-\mu}$. The first inequality follows from the convexity of $\ell_2$ norm.  The second to last inequality follows from Assumption~\ref{norm}.
\end{proof}

\begin{lemma} \label{diff_z}
	Under Assumption~\ref{var1}, we have
	\begin{equation}
	\mathbb{E} [\|   \mathbf{z}_{t}-   \bar{\mathbf{x}}_{t}\|^2] \leq \frac{ \eta^2\mu^2\sigma^2}{(1-\mu)^4K} +  \frac{\eta^2\mu^2}{(1-\mu)^3} \mathbb{E} [\sum_{t'=0}^{t-1} {\mu^{t-1-t'}}  \|\frac{1}{K}\sum_{k=1}^{K}\nabla f^{(k)}(\mathbf{x}_{t'}^{(k)})\|^2]   \ .
	\end{equation}
\end{lemma}

\begin{proof}
	Due to Eq.~(\ref{aux_z}), we can get
	\begin{equation}
	\begin{aligned}
	&  \mathbb{E} [\|   \mathbf{z}_{t}-   \bar{\mathbf{x}}_{t}\|^2] = \mathbb{E} [\|  \frac{\eta\mu}{1-\mu} \bar{\mathbf{m}}_{t}\|^2 ] =\frac{\eta^2\mu^2}{(1-\mu)^2}   \mathbb{E} [\| \frac{1}{K}\sum_{i=1}^{K}\mathbf{m}_{t}^{(k)}\|^2 ]  \ .\\
	\end{aligned}
	\end{equation}
	According to the definition of $\mathbf{m}_{t}^{(k)}$, we can get 
	\begin{equation} \label{bound_m}
	\begin{aligned}
	& \mathbb{E} [\| \frac{1}{K}\sum_{i=1}^{K}\mathbf{m}_{t}^{(k)}\|^2 ]  = \mathbb{E} [\|\sum_{t'=0}^{t-1}\mu^{t-1-t'} \frac{1}{K}\sum_{i=1}^{K}\nabla F^{(k)}(\mathbf{x}_{t'}^{(k)}; \xi_{t'}^{(k)}) \|^2] \\
	& = v_t^2 \mathbb{E} [ \|\sum_{t'=0}^{t-1} \frac{\mu^{t-1-t'}}{v_t} \frac{1}{K}\sum_{i=1}^{K} \nabla F^{(k)}(\mathbf{x}_{t'}^{(k)}, \xi_{t'}^{(k)}) \|^2] \\
	& \leq v_t \mathbb{E} [\sum_{t'=0}^{t-1} {\mu^{t-1-t'}}\| \frac{1}{K}\sum_{i=1}^{K}\nabla F^{(k)}(\mathbf{x}_{t'}^{(k)}, \xi_{t'}^{(k)}) \|^2] \\
	& \leq \frac{1}{1-\mu} \mathbb{E} [\sum_{t'=0}^{t-1} {\mu^{t-1-t'}} \Big(\frac{\sigma^2}{K}+ \|\frac{1}{K}\sum_{k=1}^{K}\nabla f^{(k)}(\mathbf{x}_{t'}^{(k)})\|^2  \Big)] \\
	& = \frac{1}{1-\mu} \frac{\sigma^2}{K} \sum_{t'=0}^{t-1} {\mu^{t-1-t'}}  +  \frac{1}{1-\mu} \mathbb{E} [\sum_{t'=0}^{t-1} {\mu^{t-1-t'}}  \|\frac{1}{K}\sum_{k=1}^{K}\nabla f^{(k)}(\mathbf{x}_{t'}^{(k)})\|^2] \\
	& \leq \frac{ \sigma^2}{(1-\mu)^2K} +  \frac{1}{1-\mu} \mathbb{E} [\sum_{t'=0}^{t-1} {\mu^{t-1-t'}}  \|\frac{1}{K}\sum_{k=1}^{K}\nabla f^{(k)}(\mathbf{x}_{t'}^{(k)})\|^2]  \ , \\
	\end{aligned}
	\end{equation}
	where  $v_t=\sum_{t'=0}^{t-1} \mu^{t-1-t'}\leq \frac{1}{1-\mu}$, the first inequality follows from the convexity of $\ell_2$ norm, the second inequality follows from Lemma~\ref{gradient_norm}.
	
	As a result, we have
	\begin{equation}
	\mathbb{E} [\|   \mathbf{z}_{t}-   \bar{\mathbf{x}}_{t}\|^2] \leq \frac{ \eta^2\mu^2\sigma^2}{(1-\mu)^4K} +  \frac{\eta^2\mu^2}{(1-\mu)^3} \mathbb{E} [\sum_{t'=0}^{t-1} {\mu^{t-1-t'}}  \|\frac{1}{K}\sum_{k=1}^{K}\nabla f^{(k)}(\mathbf{x}_{t'}^{(k)})\|^2]   \ .
	\end{equation}

\end{proof}

\begin{lemma} \label{diff_x}
	For $\forall t\geq 0$, we have
	\begin{equation}
	\sum_{k=1}^{K}  \|\mathbf{x}_{t}^{(k)} - \bar{\mathbf{x}}_{t} \|^2  \leq  \frac{2\eta^2 p^2G^2K}{(1-\mu)^2} (1+\frac{4}{\rho^2}) \ .
	\end{equation}
\end{lemma}

\begin{proof}
	Assume $r$ is the nearest communication iteration that happens before the current iteration $t$,  according to the definition of $\mathbf{x}_{t}^{(k)}  $ and $\bar{\mathbf{x}}_{t}^{(k)} $, we can get 
	\begin{equation}
	\begin{aligned}
	& \sum_{k=1}^{K}  \|\mathbf{x}_{t}^{(k)} - \bar{\mathbf{x}}_{t}\|^2   = \mathbb{E}[\|X_{t} - \bar{X}_{t}\|_F^2] \\
	&= \mathbb{E}[\|X_{r}- \bar{X}_{r}+\eta \sum_{t'=r}^{t-1}\Delta_{t'}( \frac{1}{K}\mathbf{1}\mathbf{1}^T  - \mathbf{I} )\|_F^2] \\
	& \leq  2\underbrace{\mathbb{E}[\|X_{r}- \bar{X}_{r}\|^2]}_{T_1}+2\eta^2 \underbrace{\mathbb{E}[\|\sum_{t'=r}^{t-1}\Delta_{t'}( \frac{1}{K}\mathbf{1}\mathbf{1}^T  - \mathbf{I} )\|_F^2] }_{T_2} \ .\\
	\end{aligned}
	\end{equation}
	As for $T_2$, we can get
	\begin{equation}
	\begin{aligned}
	& T_2=\mathbb{E}[\|\sum_{t'=r}^{t-1}\Delta_{t'}( \frac{1}{K}\mathbf{1}\mathbf{1}^T  - \mathbf{I} )\|_F^2] \leq \mathbb{E}[\|\sum_{t'=r}^{t-1}\Delta_{t'}\|_F^2] \leq p \sum_{t'=r}^{t-1}\mathbb{E}[\|\Delta_{t'}\|_F^2] \leq \frac{ p^2G^2K}{(1-\mu)^2}  \ ,
	\end{aligned}
	\end{equation}
	where the last inequality follows from Lemma~\ref{momentum_norm}.  When $r=0$, it is easy to get the conclustion. When $r\geq p$, 
	as for $T_1$, we can get 
	\begin{equation}
	\begin{aligned}
	& T_1 =  \mathbb{E}[\|X_{r}- \bar{X}_{r}\|^2] \\
	& =   \mathbb{E}[\|(X_{(r-1)+\frac{1}{2}} - \bar{X}_{(r-1)+\frac{1}{2}})(W-\frac{1}{K}\mathbf{1}\mathbf{1}^T)\|_F^2] \\
	& \leq \mathbb{E}[\|W-\frac{1}{K}\mathbf{1}\mathbf{1}^T\|_2^2  \|X_{(r-1)+\frac{1}{2}} - \bar{X}_{(r-1)+\frac{1}{2}}\|_F^2] \\
	& \leq (1-\rho)  \mathbb{E}[\|X_{(r-1)+\frac{1}{2}} - \bar{X}_{(r-1)+\frac{1}{2}}\|_F^2] \\
	& = (1-\rho)  \mathbb{E}[\|X_{r-p}- \bar{X}_{r-p}+\eta \sum_{t'=r-p}^{r-1}\Delta_{t'}( \frac{1}{K}\mathbf{1}\mathbf{1}^T  - \mathbf{I} )\|_F^2 ]\\
	& \leq  (1-\rho)(1+\frac{1}{c})  \mathbb{E}[\|X_{r-p}- \bar{X}_{r-p}\|_F^2 ]+(1-\rho)(1+c)\eta^2  \mathbb{E}[\| \sum_{t'=r-p}^{r-1}\Delta_{t'} \|_F^2]\\
	& \leq (1-\frac{\rho}{2})\|X_{r-p}- \bar{X}_{r-p}\|_F^2  + \frac{ 2\eta^2p^2G^2K}{\rho(1-\mu)^2} \\
	& \leq  \frac{ 4\eta^2p^2G^2K}{\rho^2(1-\mu)^2}  \ ,
	\end{aligned}
	\end{equation}
	where the second inequality follows from Lemma \ref{lemma_w_spectral}, the second to last inequality follows from $c=\frac{2}{\rho}$ and $T_2$,  the last step is obtained by the recursive expansion. At last,  by combining $T_1$ and $T_2$, we can get 
	\begin{equation}
	\sum_{k=1}^{K}  \|\mathbf{x}_{t}^{(k)} - \bar{\mathbf{x}}_{t}\|^2   \leq \frac{ 8\eta^2p^2G^2K}{\rho^2(1-\mu)^2} + \frac{2\eta^2 p^2G^2K}{(1-\mu)^2}  = \frac{2\eta^2 p^2G^2K}{(1-\mu)^2} (1+\frac{4}{\rho^2}) \ .
	\end{equation}
	
\end{proof}

Based  on the aforementioned lemmas, we are going to prove Theorem~\ref{theorem1}. 
\begin{proof}
	According to the smoothness of the loss function, we can get
	\begin{equation} \label{taylor}
	\begin{aligned}
	& \mathbb{E}[f(\mathbf{z}_{t+1}) ]\leq  \mathbb{E}[f(\mathbf{z}_{t}) + \langle \nabla  f(\mathbf{z}_{t}), \mathbf{z}_{t+1}-\mathbf{z}_{t} \rangle + \frac{L}{2} \|\mathbf{z}_{t+1}-\mathbf{z}_{t} \|^2] \\
	& =f(\mathbf{z}_{t}) - \frac{\eta}{1-\mu} \mathbb{E}[\langle \nabla  f(\mathbf{z}_{t}), \frac{1}{K}\sum_{k=1}^{K}\nabla F^{(k)}(\mathbf{x}_{t}^{(k)}; \xi_{t}^{(k)}) \rangle] \\
	& \quad + \frac{\eta^2L}{2(1-\mu)^2} \mathbb{E}[\|\frac{1}{K}\sum_{k=1}^{K}\nabla F^{(k)}(\mathbf{x}_{t}^{(k)}; \xi_{t}^{(k)})\|^2]\\
	& \leq f(\mathbf{z}_{t}) \underbrace{- \frac{\eta}{1-\mu} \mathbb{E}[\langle \nabla  f(\mathbf{z}_{t}), \frac{1}{K}\sum_{k=1}^{K}\nabla F^{(k)}(\mathbf{x}_{t}^{(k)}; \xi_{t}^{(k)}) \rangle]}_{T_1}\\
	& \quad +  \frac{\eta^2L}{2(1-\mu)^2} \mathbb{E}[\|\frac{1}{K}\sum_{k=1}^{K}\nabla f^{(k)}(\mathbf{x}_{t}^{(k)})\|^2]   +  \frac{\eta^2\sigma^2L}{2(1-\mu)^2K}  \ ,
	\end{aligned}
	\end{equation}
	where the last inequality follows from Lemma~\ref{gradient_norm}. 
	For $T_1$, we can get
	\begin{equation}
	\begin{aligned}
	& T_1 = -\frac{\eta}{1-\mu}  \langle \nabla   f(\mathbf{z}_{t}), \frac{1}{K}\sum_{k=1}^{K}\nabla  f^{(k)}(\mathbf{x}_{t}) \rangle  \\
	& = -\frac{\eta}{1-\mu}  \langle \nabla   f(\mathbf{z}_{t}) - \nabla   f(\bar{\mathbf{x}}_{t}) + \nabla   f(\bar{\mathbf{x}}_{t}), \frac{1}{K}\sum_{k=1}^{K}\nabla  f^{(k)}(\mathbf{x}_{t}) \rangle  \\
	& = -\frac{\eta}{1-\mu}  \langle \nabla   f(\mathbf{z}_{t}) - \nabla   f(\bar{\mathbf{x}}_{t}) , \frac{1}{K}\sum_{k=1}^{K}\nabla  f^{(k)}(\mathbf{x}_{t}^{(k)}) \rangle -\frac{\eta}{1-\mu}  \langle \nabla   f(\bar{\mathbf{x}}_{t}), \frac{1}{K}\sum_{k=1}^{K}\nabla  f^{(k)}(\mathbf{x}_{t}^{(k)}) \rangle \\
	& \leq \frac{(1-\mu)}{2\mu L}\|\nabla   f(\mathbf{z}_{t}) - \nabla   f(\bar{\mathbf{x}}_{t})\|^2  + \frac{\mu\eta^2L}{2(1-\mu)^3}\|\frac{1}{K}\sum_{k=1}^{K}\nabla  f^{(k)}(\mathbf{x}_{t}^{(k)}) \|^2 \\
	& \quad - \frac{\eta}{2(1-\mu)}\| \nabla   f(\bar{\mathbf{x}}_{t})\|^2  - \frac{\eta}{2(1-\mu)} \|\frac{1}{K}\sum_{k=1}^{K}\nabla  f^{(k)}(\mathbf{x}_{t}^{(k)}) \|^2 \\
	& \quad + \frac{\eta}{2(1-\mu)} \| \nabla   f(\bar{\mathbf{x}}_{t}) - \frac{1}{K}\sum_{k=1}^{K}\nabla  f^{(k)}(\mathbf{x}_{t}^{(k)}) \|^2 \\
	& \leq \frac{(1-\mu)L}{2\mu }\|   \mathbf{z}_{t}-   \bar{\mathbf{x}}_{t}\|^2  + (\frac{\mu\eta^2L}{2(1-\mu)^3} -  \frac{\eta}{2(1-\mu)})\|\frac{1}{K}\sum_{k=1}^{K}\nabla  f^{(k)}(\mathbf{x}_{t}^{(k)}) \|^2 \\
	& \quad  - \frac{\eta}{2(1-\mu)}\| \nabla   f(\bar{\mathbf{x}}_{t})\|^2   +   \frac{\eta L^2}{2K(1-\mu)}\sum_{k=1}^{K}\| \bar{\mathbf{x}}_{t} - \mathbf{x}_{t}^{(k)} \|^2 \ ,\\
	\end{aligned}
	\end{equation}
	where the first inequality follows from $\mathbf{x}^T\mathbf{y}=\frac{1}{2}\|\mathbf{x}\|^2 + \frac{1}{2}\|\mathbf{y}\|^2$ and the last inequality follows Assumption~\ref{smooth}. 
	
	By putting $T_1$ into Eq.~(\ref{taylor}), we can get
	\begin{equation}
	\begin{aligned}
	& f(\mathbf{z}_{t+1})  \leq f(\mathbf{z}_{t})  - \frac{\eta}{2(1-\mu)}\| \nabla   f(\bar{\mathbf{x}}_{t})\|^2   + \frac{(1-\mu)L}{2\mu }\|   \mathbf{z}_{t}-   \bar{\mathbf{x}}_{t}\|^2 +   \frac{\eta L^2}{2K(1-\mu)}\sum_{k=1}^{K}\| \bar{\mathbf{x}}_{t} - \mathbf{x}_{t} ^{(k)}\|^2 \\
	&  \quad + (\frac{\mu\eta^2L}{2(1-\mu)^3} -  \frac{\eta}{2(1-\mu)} + \frac{\eta^2L}{2(1-\mu)^2})\|\frac{1}{K}\sum_{k=1}^{K}\nabla  f^{(k)}(\mathbf{x}_{t}^{(k)}) \|^2  + \frac{\eta^2\sigma^2L}{2K(1-\mu)^2}  \\
	& \leq f(\mathbf{z}_{t})  - \frac{\eta}{2(1-\mu)}\| \nabla   f(\bar{\mathbf{x}}_{t})\|^2     +   \frac{\eta^3 p^2G^2L^2}{(1-\mu)^3} (1+\frac{4}{\rho^2}) + \frac{ \mu\eta^2\sigma^2L}{2(1-\mu)^3K} \\
	& \quad +  \frac{\mu\eta^2 L}{2(1-\mu)^2} \mathbb{E} [\sum_{t'=0}^{t-1} {\mu^{t-1-t'}}  \|\frac{1}{K}\sum_{k=1}^{K}\nabla f^{(k)}(\mathbf{x}_{t'}^{(k)})\|^2]  \\
	& \quad  + (\frac{\mu\eta^2L}{2(1-\mu)^3} -  \frac{\eta}{2(1-\mu)} + \frac{\eta^2L}{2(1-\mu)^2})\|\frac{1}{K}\sum_{k=1}^{K}\nabla  f^{(k)}(\mathbf{x}_{t}^{(k)}) \|^2  + \frac{\eta^2\sigma^2L}{2K(1-\mu)^2}   \ ,\\
	\end{aligned}
	\end{equation}
	where the second inequality follows from Lemma~\ref{diff_z} and Lemma~\ref{diff_x}. By summing $t$ from $0$ to $T-1$, we can get
	\begin{equation}
	\begin{aligned}
	&  f(\mathbf{z}_{T}) \leq f(\mathbf{z}_{0})  - \frac{\eta}{2(1-\mu)}\sum_{t=0}^{T-1}\| \nabla   f(\bar{\mathbf{x}}_{t})\|^2     +   \frac{\eta^3 p^2G^2L^2}{(1-\mu)^3} (1+\frac{4}{\rho^2}) T+ \frac{ \mu\eta^2\sigma^2LT}{2(1-\mu)^3K} \\
	& \quad +  \frac{\mu\eta^2 L}{2(1-\mu)^2} \sum_{t=0}^{T-1}\mathbb{E} [\sum_{t'=0}^{t-1} {\mu^{t-1-t'}}  \|\frac{1}{K}\sum_{k=1}^{K}\nabla f^{(k)}(\mathbf{x}_{t'}^{(k)})\|^2]  \\
	& \quad  + (\frac{\mu\eta^2L}{2(1-\mu)^3} -  \frac{\eta}{2(1-\mu)} + \frac{\eta^2L}{2(1-\mu)^2})\sum_{t=0}^{T-1}\|\frac{1}{K}\sum_{k=1}^{K}\nabla  f^{(k)}(\mathbf{x}_{t}^{(k)}) \|^2  + \frac{\eta^2\sigma^2LT}{2K(1-\mu)^2}  \\
	& \leq f(\mathbf{z}_{0})  - \frac{\eta}{2(1-\mu)}\sum_{t=0}^{T-1}\| \nabla   f(\bar{\mathbf{x}}_{t})\|^2     +   \frac{\eta^3 p^2G^2L^2}{(1-\mu)^3} (1+\frac{4}{\rho^2}) T+ \frac{ \mu\eta^2\sigma^2LT}{2(1-\mu)^3K} \\
	& \quad  + (\frac{\mu\eta^2L}{(1-\mu)^3} -  \frac{\eta}{2(1-\mu)} + \frac{\eta^2L}{2(1-\mu)^2})\sum_{t=0}^{T-1}\|\frac{1}{K}\sum_{k=1}^{K}\nabla  f^{(k)}(\mathbf{x}_{t}^{(k)}) \|^2  + \frac{\eta^2\sigma^2LT}{2K(1-\mu)^2}  \ , \\
	\end{aligned}
	\end{equation}
	where the last inequality  follows from $\sum_{t=0}^{T-1}\mathbb{E} [\sum_{t'=0}^{t-1} {\mu^{t-1-t'}}  \|\frac{1}{K}\sum_{k=1}^{K}\nabla f^{(k)}(\mathbf{x}_{t'}^{(k)})\|^2] \leq \frac{1}{1-\mu}\sum_{t=0}^{T-1}\|\frac{1}{K}\sum_{k=1}^{K}\nabla f^{(k)}(\mathbf{x}_{t}^{(k)})\|^2$.
	
	Then, we can get 
	\begin{equation}
	\begin{aligned}
	&  \frac{1}{T}\sum_{t=0}^{T-1}\| \nabla   f(\bar{\mathbf{x}}_{t})\|^2  \leq  \frac{2(1-\mu)(f(\mathbf{z}_{0})  -  f(\mathbf{z}_{T}) )}{\eta T}+ (\frac{2\mu\eta L}{(1-\mu)^2} - 1+\frac{\eta L}{1-\mu}) \frac{1}{T}\sum_{t=0}^{T-1}\|\frac{1}{K}\sum_{k=1}^{K}\nabla  f^{(k)}(\mathbf{x}_{t}) \|^2 \\
	&  \quad + \frac{\mu\eta \sigma^2L}{(1-\mu)^2K}  + \frac{2\eta^2 p^2G^2L^2}{(1-\mu)^2} (1+\frac{4}{\rho^2})   + \frac{\eta\sigma^2L}{K(1-\mu)} \\
	& \leq  \frac{2(1-\mu)(f(\mathbf{z}_{0})  -  f(\mathbf{z}_{t+1}) )}{\eta T}+ \frac{\mu\eta \sigma^2L}{(1-\mu)^2K}  +   \frac{2\eta^2 p^2G^2L^2}{(1-\mu)^2} (1+\frac{4}{\rho^2})   + \frac{\eta\sigma^2L}{(1-\mu)K}  \ ,\\
	\end{aligned}
	\end{equation}
	where the last inequality follows from $\eta<(1-\mu)^2/(2L)$.
	Therefore,
	\begin{equation}
	\begin{aligned}
	& \frac{1}{T}\sum_{t=0}^{T-1} \| \nabla   f(\bar{\mathbf{x}}_{t})\|^2  \leq \frac{2(1-\mu)(f(\mathbf{x}_{0})  -  f^* )}{\eta T}+ \frac{\mu\eta \sigma^2L}{(1-\mu)^2K} +  \frac{2\eta^2 p^2G^2L^2}{(1-\mu)^2} (1+\frac{4}{\rho^2})   + \frac{\eta\sigma^2L}{(1-\mu)K} \ . 
	\end{aligned}
	\end{equation}
	
\end{proof}

\subsection{Proof of  Theorem~\ref{theorem2}} \label{proof_theorem_2}
To prove Theorem~\ref{theorem2}, we  first introduce the following notation  for the auxiliary  variable. 
\begin{equation} \label{mat_2}
\hat{X}_t = [\hat{\mathbf{x}}_t^{(1)}, \hat{\mathbf{x}}_t^{(2)}, \cdots, \hat{\mathbf{x}}_t^{(K)}] \in \mathbb{R}^{d\times K}  \ .
\end{equation}
By combinning Eq.~(\ref{mat_1}) and  Eq.~(\ref{mat_2}), we can represent Algorithm~\ref{alg_dec_sgdm_com} as follows:
\begin{equation}
\begin{aligned}
&  X_{t+\frac{1}{2}}  = X_t - \eta \Delta_t  \ ,\\
& X_{t+1}= \begin{cases}
X_{t+\frac{1}{2}}  ,  & \mod($t+1$, $p$)\neq 0 \\
X_{t+\frac{1}{2}} + \gamma \hat{X}_{t}(W-I) , & \mod($t+1$, $p$)=0\\
\end{cases}   \ ,\\
& \hat{X}_{t+1}= \begin{cases}
\hat{X}_{t}  ,  & \mod($t+1$, $p$)\neq 0 \\
\hat{X}_{t} + Q(X_{t+1}- \hat{X}_{t}) , & \mod($t+1$, $p$)=0\\
\end{cases}  \ . \\
\end{aligned}
\end{equation}
Based on these notations,  at two consecutive communication rounds $r_1=r_0+p$, we can get  the  following   representation
\begin{equation} \label{matrix_compress}
\begin{aligned}
& X_{(r_1-1)+\frac{1}{2}} = X_{r_0} -\eta \sum_{t=r_0}^{r_1-1}\Delta_{t} \ ,\\
& X_{r_1} = X_{(r_1-1)+\frac{1}{2}} +\gamma \hat{X}_{r_0}(W-I)  \ ,\\ 
& \hat{X}_{r_1} = \hat{X}_{r_0} + Q(X_{r_1}- \hat{X}_{r_0}) \ .
\end{aligned}
\end{equation}


Moreover, since $W$ is a doubly stochastic matrix, at the communication step, we can get 
\begin{equation} \label{x_bar_compress}
\begin{aligned}
& \bar{X}_{r_1} =  \bar{X}_{(r_1-1)+\frac{1}{2}} +\gamma \hat{X}_{r_0}(W-I)\frac{1}{K}\mathbf{1}\mathbf{1}^T =  \bar{X}_{(r_1-1)+\frac{1}{2}} = \bar{X}_{r_1-1} -\eta \Delta_{r_1-1}\frac{1}{K}\mathbf{1}\mathbf{1}^T  \ .
\end{aligned}
\end{equation}
As a result, at each iteration $t\geq0$, we have
\begin{equation} \label{compress_iter}
\bar{X}_{t+1}  = \bar{X}_{t} -\eta \Delta_t\frac{1}{K}\mathbf{1}\mathbf{1}^T  \ .
\end{equation}
Then, we can use the same auxiliary variable $\mathbf{z}_t$  in  the previous subsection. In the following, we are going to introduce some fundamental lemmas for proving Theorem~\ref{theorem2}.  
\begin{lemma} \label{lemma_x_diff_2}
	Under Assumption~\ref{lemma_w_spectral}--\ref{norm}, we have
	\begin{equation}
	\begin{aligned}
	&\sum_{k=1}^{K}  \|\mathbf{x}_{t}^{(k)} - \bar{\mathbf{x}}_{t} \|^2  \leq \frac{4\eta^2 p^2G^2K}{(1-\mu)^2}(1+\frac{4}{\alpha^2}) \ ,\\
	\end{aligned}
	\end{equation}
	where $\alpha=\frac{\rho^2\delta}{82}<1$. 
\end{lemma}
\begin{proof}
	Assume $r$ is the nearest communication iteration that happens before the current iteration $t$, then we can get 
	\begin{equation}
	\begin{aligned}
	&\mathbb{E}[ \| X_{t} - \bar{X}_{t} \|_F^2] \\
	& \leq \mathbb{E}[ \| X_{t} - \bar{X}_{t} \|_F^2] +  \mathbb{E}[ \| X_{t} - \hat{X}_{t} \|_F^2] \\
	&= \mathbb{E}[\|X_{r}- \bar{X}_{r}+\eta \sum_{t'=r}^{t-1}\Delta_{t'}( \frac{1}{K}\mathbf{1}\mathbf{1}^T  - \mathbf{I} )\|_F^2]  + \mathbb{E}[\|X_{r} - \hat{X}_{r} - \eta \sum_{t'=r}^{t-1}\Delta_{t'}\|_F^2]  \\
	& \leq  2\mathbb{E}[\|X_{r}- \bar{X}_{r} \|_F^2]+ 2  \mathbb{E}[\|X_{r} - \hat{X}_{r} \|_F^2]\\
	& \quad +2\eta^2\mathbb{E}[\|\sum_{t'=r}^{t-1}\Delta_{t'}( \frac{1}{K}\mathbf{1}\mathbf{1}^T  - \mathbf{I} )\|_F^2 ]  +  2\eta^2\mathbb{E}[\| \sum_{t'=r}^{t-1}\Delta_{t'}\|_F^2]  \\
	& \leq 2\mathbb{E}[\|X_{r}- \bar{X}_{r} \|_F^2]+ 2  \mathbb{E}[\|X_{r} - \hat{X}_{r} \|_F^2]   +  4\eta^2\mathbb{E}[\| \sum_{t'=r}^{t-1}\Delta_{t'}\|_F^2]  \\
	& \leq 2\mathbb{E}[\|X_{r}- \bar{X}_{r} \|_F^2]+ 2  \mathbb{E}[\|X_{r} - \hat{X}_{r} \|_F^2] +  \frac{ 4\eta^2 p^2G^2K}{(1-\mu)^2} \ .\\
	\end{aligned}
	\end{equation}
	When $r=0$, the conclusion is easy to get. In the  following, we are going to bound the first two terms respectively given $r\geq p$. Here, the analysis for these two terms is adapted from \cite{koloskova2019decentralizedcon}. But they  are  different because our method  involves the multiple local computation. Thus, our method is more complicated. In detail, for the first term, we can get
	\begin{equation}
	\begin{aligned}
	& \|X_{r} - \bar{X}_{r}\|_F^2  \\
	& =  \| X_{(r-1)+\frac{1}{2}} +\gamma \hat{X}_{r-p}(W-I) -\bar{X}_{(r-1)+\frac{1}{2}} \|_F^2 \\
	& = \| X_{(r-1)+\frac{1}{2}} +\gamma \hat{X}_{r-p}(W-I) -\bar{X}_{(r-1)+\frac{1}{2}} - \gamma\bar{X}_{(r-1)+\frac{1}{2}} (W-I)\|_F^2 \\
	& = \| (X_{(r-1)+\frac{1}{2}}- \bar{X}_{(r-1)+\frac{1}{2}})(I+\gamma(W-I)) + \gamma(\hat{X}_{r-1} - X_{(r-1)+\frac{1}{2}})(W-I)\|_F^2\\
	& \leq (1+c_1) \|(X_{(r-1)+\frac{1}{2}}- \bar{X}_{(r-1)+\frac{1}{2}})(I+\gamma(W-I)) \|_F^2 + (1+\frac{1}{c_1}) \| \gamma(\hat{X}_{r-p} - X_{(r-1)+\frac{1}{2}})(W-I)\|_F^2 \\
	& \leq (1+c_1)(1-\gamma \rho)^2\|X_{(r-1)+\frac{1}{2}}- \bar{X}_{(r-1)+\frac{1}{2}} \|_F^2 + (1+\frac{1}{c_1})\gamma^2\beta^2\| \hat{X}_{r-p} - X_{(r-1)+\frac{1}{2}}\|_F^2  \ ,
	\end{aligned}
	\end{equation}
	where the first equality is due to Eq.~(\ref{matrix_compress}) and Eq.~(\ref{x_bar_compress}), the last step follows from  $\|W\|_2\leq \beta\triangleq\max_i\{1-\lambda_i\} $ and 
	\begin{equation}
	\begin{aligned}
	& \|(X_{(r-1)+\frac{1}{2}}- \bar{X}_{(r-1)+\frac{1}{2}})(I+\gamma(W-I)) \|_F\\
	& \leq (1-\gamma)\|X_{(r-1)+\frac{1}{2}}- \bar{X}_{(r-1)+\frac{1}{2}}\|_F + \gamma \|(X_{(r-1)+\frac{1}{2}}- \bar{X}_{(r-1)+\frac{1}{2}})W\|_F\\
	& =  (1-\gamma)\|X_{(r-1)+\frac{1}{2}}- \bar{X}_{(r-1)+\frac{1}{2}}\|_F + \gamma\|(X_{(r-1)+\frac{1}{2}}- \bar{X}_{(r-1)+\frac{1}{2}})(W-\frac{\mathbf{1}\mathbf{1}^T}{K})\|_F \\
	& \leq (1-\gamma\rho)\|(X_{(r-1)+\frac{1}{2}}- \bar{X}_{(r-1)+\frac{1}{2}})\|_F \ .
	\end{aligned}
	\end{equation}
	
	For the second term, we can get
	\begin{equation}
	\begin{aligned}
	&  \|X_{r} - \hat{X}_{r}\|_F^2  \\
	& = \|X_{r} - \hat{X}_{r-p}-Q(X_{r}- \hat{X}_{r-p})\|_F^2 \\
	& \leq (1-\delta)\|X_{r} - \hat{X}_{r-p}\|_F^2 \\
	& =  (1-\delta) \|  X_{(r-1)+\frac{1}{2}} +\gamma \hat{X}_{r-p}(W-I) - \hat{X}_{r-p}\|_F^2 \\
	& =  (1-\delta) \|  X_{(r-1)+\frac{1}{2}} - \hat{X}_{r-p}(I - \gamma(W-I)) -\gamma\bar{X}_{(r-1)+\frac{1}{2}} (W-I)\|_F^2 \\
	& = (1-\delta)  \|(X_{(r-1)+\frac{1}{2}} - \hat{X}_{r-p})(I - \gamma(W-I)) +\gamma(X_{(r-1)-\frac{1}{2}}-\bar{X}_{(r-1)+\frac{1}{2}} )(W-I)\|_F^2 \\
	& \leq (1-\delta)(1+c_2)\|(X_{(r-1)+\frac{1}{2}} - \hat{X}_{r-p})(I - \gamma(W-I))\|_F^2 \\
	& \quad + (1-\delta)(1+\frac{1}{c_2})\|\gamma(X_{(r-1)+\frac{1}{2}}-\bar{X}_{(r-1)+\frac{1}{2}} )(W-I)\|_F^2 \\
	& \leq (1-\delta)(1+c_2)(1+\gamma\beta)^2\|X_{(r-1)+\frac{1}{2}} - \hat{X}_{r-p}\|_F^2 \\
	& \quad + (1-\delta)\gamma^2\beta^2(1+\frac{1}{c_2})\|X_{(r-1)+\frac{1}{2}}-\bar{X}_{(r-1)+\frac{1}{2}} \|_F^2  \ ,
	\end{aligned}
	\end{equation}
	where the second step follows that $Q$ is $\delta$-contraction.  	Combining these two terms, we have
	\begin{equation}
	\begin{aligned}
	& \|X_{r} - \bar{X}_{r}\|_F^2 + \|X_{r} - \hat{X}_{r}\|_F^2  \\
	& \leq ((1+c_1)(1-\gamma \rho)^2 + (1-\delta)\gamma^2\beta^2(1+\frac{1}{c_2}))\|X_{(r-1)+\frac{1}{2}}- \bar{X}_{(r-1)+\frac{1}{2}} \|_F^2 \\
	& +( (1+\frac{1}{c_1})\gamma^2\beta^2 + (1-\delta)(1+c_2)(1+\gamma\beta)^2)\| \hat{X}_{r-p} - X_{(r-1)+\frac{1}{2}}\|_F^2  \ .
	\end{aligned}
	\end{equation}
	Similar with \cite{koloskova2019decentralizedcon}, by setting $c_{1}=\frac{\gamma \rho}{2}$, $c_{2}=\frac{\delta}{2}$, and $\gamma=\frac{\rho \delta}{16 \rho+\rho^{2}+4 \beta^{2}+2 \rho \beta^{2}-8 \rho \delta}$, we can get  $\alpha=\frac{\rho^2\delta}{82}<1$ such that
	\begin{equation}
	\begin{aligned}
	& \|X_{r} - \bar{X}_{r}\|_F^2 + \|X_{r} - \hat{X}_{r}\|_F^2 \\
	&  \leq (1-\alpha) \|X_{(r-1)+\frac{1}{2}}- \bar{X}_{(r-1)+\frac{1}{2}} \|_F^2  + (1-\alpha)\| \hat{X}_{r-p} - X_{(r-1)+\frac{1}{2}}\|_F^2 \\
	& = (1-\alpha) \mathbb{E}[\|X_{r-p}- \bar{X}_{r-p}+\eta \sum_{t=r-p}^{r-1}\Delta_{t}( \frac{1}{K}\mathbf{1}\mathbf{1}^T  - \mathbf{I} )\|_F^2 ] \\
	& \quad + (1-\alpha) \mathbb{E}[\| \hat{X}_{r-p} -X_{r-p} +\eta \sum_{t=r-p}^{r-1}\Delta_{t}\|_F^2] \\
	& \leq (1-\alpha)\Big((1+\frac{1}{c})\mathbb{E}[\|X_{r-p}- \bar{X}_{r-p}\|_F^2 + \| X_{r-p} -\hat{X}_{r-p} \|_F^2]\\
	& \quad +(1+c)\eta^2 \mathbb{E}[\|\sum_{t=r-p}^{r-1}\Delta_{t}( \frac{1}{K}\mathbf{1}\mathbf{1}^T  - \mathbf{I} )\|_F^2+\|\sum_{t=r-p}^{r-1}\Delta_{t}\|_F^2]\Big) \\
	& \leq  (1-\alpha)\Big((1+\frac{1}{c})\mathbb{E}[\|X_{r-p}- \bar{X}_{r-p}\|_F^2 + \| X_{r-p} -\hat{X}_{r-p} \|_F^2] + 2\eta^2(1+c) \frac{ p^2G^2K}{(1-\mu)^2}\Big) \\
	& \leq (1-\frac{\alpha}{2})\Big(\mathbb{E}[\|X_{r-p}- \bar{X}_{r-p}\|_F^2 + \| X_{r-p} -\hat{X}_{r-p} \|_F^2] \Big)  + \frac{4\eta^2p^2G^2K}{\alpha(1-\mu)^2}\\
	& \leq  \frac{8\eta^2p^2G^2K}{\alpha^2(1-\mu)^2} \ ,
	\end{aligned}
	\end{equation}
	where the second to  last step follows from $c=2/\alpha$, the last step is due to the recursive expansion. Therefore, we have
	\begin{equation}
	\begin{aligned}
	&\sum_{k=1}^{K}  \|\mathbf{x}_{t}^{(k)} - \bar{\mathbf{x}}_{t} \|^2  = \mathbb{E}[ \| X_{t} - \bar{X}_{t} \|_F^2] \leq \frac{4\eta^2 p^2G^2K}{(1-\mu)^2}(1+\frac{4}{\alpha^2}) \ .\\
	\end{aligned}
	\end{equation}
	
\end{proof}

Now, we are going to prove Theorem~\ref{theorem2}. 
\begin{proof}
	In terms of  Eq.~(\ref{compress_iter}), we can employ the similar idea as Theorem~\ref{theorem1} to get 
	\begin{equation}
	\begin{aligned}
	& f(\mathbf{z}_{t+1})  \leq f(\mathbf{z}_{t})  - \frac{\eta}{2(1-\mu)}\| \nabla   f(\bar{\mathbf{x}}_{t})\|^2   + \frac{(1-\mu)L}{2\mu }\|   \mathbf{z}_{t}-   \bar{\mathbf{x}}_{t}\|^2 +   \frac{\eta L^2}{2K(1-\mu)}\sum_{k=1}^{K}\| \bar{\mathbf{x}}_{t} - \mathbf{x}_{t}^{(k)}\|^2 \\
	&  \quad + (\frac{\mu\eta^2L}{2(1-\mu)^3} -  \frac{\eta}{2(1-\mu)} + \frac{\eta^2L}{2(1-\mu)^2})\|\frac{1}{K}\sum_{k=1}^{K}\nabla  f^{(k)}(\mathbf{x}_{t}^{(k)}) \|^2  + \frac{\eta^2\sigma^2L}{2K(1-\mu)^2}  \ . \\
	\end{aligned}
	\end{equation}
	By putting  Lemma~\ref{lemma_x_diff_2} and Lemma~\ref{diff_z} into  the above inequality and summing $t$ from $0$ to $T-1$, we can get 
	\begin{equation}
	\begin{aligned}
	&  \frac{\eta}{2(1-\mu)}\sum_{t=0}^{T-1}\| \nabla   f(\bar{\mathbf{x}}_{t})\|^2  \leq   f(\mathbf{z}_{0})  -  f(\mathbf{z}_{T})   +   \frac{2\eta^3 p^2G^2L^2}{(1-\mu)^3} (1+\frac{4}{\alpha^2})T + \frac{ \mu\eta^2\sigma^2L}{2(1-\mu)^3K}T \\
	& \quad  + (\frac{\mu\eta^2L}{(1-\mu)^3} -  \frac{\eta}{2(1-\mu)} + \frac{\eta^2L}{2(1-\mu)^2})\sum_{t=0}^{T-1}\|\frac{1}{K}\sum_{k=1}^{K}\nabla  f^{(k)}(\mathbf{x}_{t}^{(k)}) \|^2  + \frac{\eta^2\sigma^2L}{2K(1-\mu)^2}T  \ . \\
	\end{aligned}
	\end{equation}
	By setting $\eta<(1-\mu)^2/(2L)$, we can get 
	\begin{equation}
	\begin{aligned}
	& \frac{1}{T}\sum_{t=0}^{T-1} \| \nabla   f(\bar{\mathbf{x}}_{t})\|^2  \leq \frac{2(1-\mu)(f(\mathbf{x}_{0})  -  f^*)}{\eta T}+ \frac{\mu\eta \sigma^2L}{(1-\mu)^2K} + \frac{4\eta^2p^2G^2L^2}{(1-\mu)^2} (1+\frac{4}{\alpha^2})   + \frac{\eta\sigma^2L}{(1-\mu)K}  \ .
	\end{aligned}
	\end{equation}
\end{proof}

\end{document}